\documentclass[letterpaper,11pt]{article}
\usepackage[margin=1in]{geometry}  

\usepackage{caption,subcaption}
\usepackage{listings}
\usepackage{amsmath}
\usepackage{amsthm}
\usepackage{tikz}
\usepackage{caption}
\usepackage{array}
\usepackage{mdwmath}
\usepackage{multirow}
\usepackage{mdwtab}
\usepackage{eqparbox}
\usepackage{multicol}
\usepackage{amsfonts}
\usepackage{tikz}
\usepackage{multirow,bigstrut,threeparttable}
\usepackage{amsthm}
\usepackage{array}
\usepackage{bbm}
\usepackage{epstopdf}
\usepackage{mdwmath}
\usepackage{mdwtab}
\usepackage{eqparbox}
\usetikzlibrary{topaths,calc}
\usepackage{latexsym}
\usepackage{cite}
\usepackage{amssymb}
\usepackage{bm}
\usepackage{amssymb}
\usepackage{graphicx}
\usepackage{mathrsfs}
\usepackage{epsfig}
\usepackage{psfrag}
\usepackage{setspace}
\usepackage[
            CJKbookmarks=true,
            bookmarksnumbered=true,
            bookmarksopen=true,
            colorlinks=true,
            citecolor=red,
            linkcolor=blue,
            anchorcolor=red,
            urlcolor=blue
            ]{hyperref}
\usepackage[ruled]{algorithm2e}
\usepackage{algpseudocode}
\usepackage{stfloats}

\usepackage{mathtools}
\usepackage{blkarray}
\usepackage{multirow,bigdelim,dcolumn,booktabs}

\usepackage{xparse}
\usepackage{tikz}
\usetikzlibrary{calc}
\usetikzlibrary{decorations.pathreplacing,matrix,positioning}

\usepackage[T1]{fontenc}
\usepackage[utf8]{inputenc}
\usepackage{mathtools}
\usepackage{blkarray, bigstrut}
\usepackage{gauss}

\newcommand*{\BraceAmplitude}{0.4em}%
\newcommand*{\VerticalOffset}{0.5ex}%
\newcommand*{\HorizontalOffset}{0.0em}%
\newcommand*{\blocktextwid}{3.0cm}%
\NewDocumentCommand{\InsertLeftBrace}{%
	O{} 
	O{\HorizontalOffset,\VerticalOffset} 
	O{\blocktextwid} 
	m   
	m   
	m   
}{%
	\begin{tikzpicture}[overlay,remember picture]
	\coordinate (Brace Top)    at ($(#4.north) + (#2)$);
	\coordinate (Brace Bottom) at ($(#5.south) + (#2)$);
	\draw [decoration={brace, amplitude=\BraceAmplitude}, decorate, thick, draw=black, #1]
	(Brace Bottom) -- (Brace Top) 
	node [pos=0.5, anchor=east, align=left, text width=#3, color=black, xshift=\BraceAmplitude] {#6};
	\end{tikzpicture}%
}%
\NewDocumentCommand{\InsertRightBrace}{%
	O{} 
	O{\HorizontalOffset,\VerticalOffset} 
	O{\blocktextwid} 
	m   
	m   
	m   
}{%
	\begin{tikzpicture}[overlay,remember picture]
	\coordinate (Brace Top)    at ($(#4.north) + (#2)$);
	\coordinate (Brace Bottom) at ($(#5.south) + (#2)$);
	\draw [decoration={brace, amplitude=\BraceAmplitude}, decorate, thick, draw=black, #1]
	(Brace Top) -- (Brace Bottom) 
	node [pos=0.5, anchor=west, align=left, text width=#3, color=black, xshift=\BraceAmplitude] {#6};
	\end{tikzpicture}%
}%
\NewDocumentCommand{\InsertTopBrace}{%
	O{} 
	O{\HorizontalOffset,\VerticalOffset} 
	O{\blocktextwid} 
	m   
	m   
	m   
}{%
	\begin{tikzpicture}[overlay,remember picture]
	\coordinate (Brace Top)    at ($(#4.west) + (#2)$);
	\coordinate (Brace Bottom) at ($(#5.east) + (#2)$);
	\draw [decoration={brace, amplitude=\BraceAmplitude}, decorate, thick, draw=black, #1]
	(Brace Top) -- (Brace Bottom) 
	node [pos=0.5, anchor=south, align=left, text width=#3, color=black, xshift=\BraceAmplitude] {#6};
	\end{tikzpicture}%
}%

\usetikzlibrary{patterns}

\definecolor{cof}{RGB}{219,144,71}
\definecolor{pur}{RGB}{186,146,162}
\definecolor{greeo}{RGB}{91,173,69}
\definecolor{greet}{RGB}{52,111,72}

\theoremstyle{plain}
\newtheorem{theorem}{Theorem}
\newtheorem{example}{Example}
\newtheorem{lemma}{Lemma}
\newtheorem{remark}{Remark}

\newtheorem{definition}{Definition}

\def \bP {\mathbb{P}}
\def \bE {\mathbb{E}}
\def \bR {\mathbb{R}}

\def \var {\mathsf{Var}}
\def\1{\mathbbm{1}}

\usepackage{xspace}

\newcommand{\stepa}[1]{\overset{\rm (a)}{#1}}
\newcommand{\stepb}[1]{\overset{\rm (b)}{#1}}
\newcommand{\stepc}[1]{\overset{\rm (c)}{#1}}
\newcommand{\stepd}[1]{\overset{\rm (d)}{#1}}

\definecolor{myblue}{rgb}{.8, .8, 1}
\definecolor{mathblue}{rgb}{0.2472, 0.24, 0.6} 
\definecolor{mathred}{rgb}{0.6, 0.24, 0.442893}
\definecolor{mathyellow}{rgb}{0.6, 0.547014, 0.24}

\newcommand{\red}{\color{red}}

\newcommand{\calB}{{\mathcal{B}}}

\newcommand{\calF}{{\mathcal{F}}}

\newcommand{\calN}{{\mathcal{N}}}

\newcommand{\calP}{{\mathcal{P}}}

\usepackage{cleveref}
\crefname{lemma}{Lemma}{Lemmas}
\Crefname{lemma}{Lemma}{Lemmas}
\crefname{thm}{Theorem}{Theorems}
\Crefname{thm}{Theorem}{Theorems}
\crefformat{equation}{(#2#1#3)}

\begin{document}

\title{Learning to Bid Optimally and Efficiently \\
	in Adversarial First-price Auctions}
\author{Yanjun Han, Zhengyuan Zhou, Aaron Flores, Erik Ordentlich, Tsachy Weissman\thanks{Y. Han is with the Courant Institute of Mathematical Sciences and the Center for Data Science, New York University, email: \url{yanjunhan@nyu.edu}. Z. Zhou is with the Stern School of Business, New York University, email: \url{zzhou@stern.nyu.edu}. T. Weissman is with the Department of Electrical Engineering, Stanford University, email: \url{tsachy@stanford.edu}. A. Flores and E. Ordentlich are with Yahoo! Research at Verizon Media, email: \url{{aaron.flores,eord}@verizonmedia.com}. This project was supported in part by NSF awards CCF-2106467 and CCF-2106508, and by the Yahoo Faculty Research and Engagement Program.}}
\maketitle
\begin{abstract}
First-price auctions have very recently swept the online advertising industry, replacing second-price auctions
as the predominant auction mechanism on many platforms. This shift has brought forth important challenges for a bidder: how should one bid in a first-price auction, where unlike in second-price auctions, it is no longer optimal to bid one's private value truthfully and hard to know the others' bidding behaviors? In this paper, we take an online learning angle and address the fundamental problem of learning to bid in repeated first-price auctions, where both the bidder's private valuations and other bidders' bids can be arbitrary. We develop the first minimax optimal online bidding algorithm that achieves an $\widetilde{O}(\sqrt{T})$ regret when competing with the set of all Lipschitz bidding policies, a strong oracle that contains a rich set of bidding strategies. 
This novel algorithm is built on the insight that the presence of a good expert can be leveraged to improve performance, as well as an original hierarchical expert-chaining structure, both of which could be of independent interest in online learning.
Further, by exploiting the product structure that exists in the problem, we modify this algorithm--in its vanilla form statistically optimal but computationally infeasible--to a computationally efficient and space efficient algorithm that also retains the same $\widetilde{O}(\sqrt{T})$ minimax optimal regret guarantee.
Additionally, through an impossibility result, we highlight that one is unlikely to compete this favorably with a stronger oracle (than the considered Lipschitz bidding policies).
Finally, we test our algorithm on three real-world first-price auction datasets obtained from Verizon Media and demonstrate our algorithm's superior performance compared to several existing bidding algorithms. 
\end{abstract}

\tableofcontents

\section{Introduction}

With the tailwind of e-commerce sweeping across industries, online advertising has had and continues to exert an enormous economic impact: in 2019, businesses in US alone~\cite{news1} have spent more than 129 billion dollars, a number that has been fast growing and projected to increase, on digital ads, beating the combined amount spent via traditional advertising channels (TV, radio, and newspapers etc.) for the first time. A core and arguably the most economically impactful element of online advertising is online auctions, where publishers sell advertising spaces (i.e. slots) to advertisers (a.k.a. bidders) through auctions held on online platforms known as \emph{ad exchanges}.

In the past, second-price auctions~\cite{vickrey1961counterspeculation} have been the predominant auction mechanism on various platforms {\cite{lucking2000vickrey,klemperer2004auctions,lucking2007pennies}}, mainly because of its truthful nature (as well as its social welfare maximization property): it is in each bidder's best interest to bid one's own private value truthfully. However, very recently there has been an industry-wide shift from second-price auctions to first-price auctions, for a number of reasons such as enhanced transparency (where the seller no longer has the ``last look'' advantage), an increased revenue of the seller (and the exchange), and the increasing popularity of header bidding which allows multiple platforms to simultaneously bid on the same inventory \cite{news2, Google}. 
Driven by these advantages, several well-known exchanges including AppNexus, Index Exchange, and OpenX, rolled out first-price auctions in 2017 \cite{Exchange}, and Google Ad Manager, by far the largest online auction platform, moved to first-price auctions completely at the end of 2019 \cite{Google2}. As such, this shift has brought forth important challenges, which do not exist for second-price auctions prior to the shift, to bidders since the optimal bidding strategy in first-price auctions is no longer truthful. This leads to a pressingly challenging question that is unique in first-price auctions: how should a bidder (adaptively) bid to maximize her cumulative payoffs when she needs to bid repeatedly facing a first-price auction?

Needless to say, the rich literature on auctions theory has studied several related aspects of the bidding problem. At a high level, the existing literature can be divided into two broad approaches. The first (and also the more traditional) approach takes a game-theoretic view of auctions: it adopts a Bayesian setup where the bidders have perfect or partial knowledge of each other's private valuations, modeled as probability distributions. Proceeding from this standpoint, the pure or mixed Nash equilibria that model optimal outcomes of the auction can be derived  \cite{wilson1969communications,myerson1981optimal,riley1981optimal}. Despite the elegance offered by this approach, an important drawback is that in practice, a bidder is unlikely to have a good model of other bidders' private valuations~\cite{wilson1985game}, thereby making the bidding strategy derived based on this premise infeasible to implement. 

Motivated to mitigate this drawback, the second (and much more recent) approach is based on online learning in repeated auctions, where a participating bidder can learn to adaptively make bids by incorporating the past data, with the goal of maximizing cumulative reward during a given time horizon. Under this framework, modeling repeated auctions as bandits has inspired a remarkable line of work~\cite{blum2004online,devanur2009price,babaioff2014characterizing,medina2014learning,cesa2014regret,babaioff2015truthful,mohri2015revenue,weed2016online,cai2017learning,golrezaei2019dynamic,roughgarden2019minimizing,zhao2020online} that lies at the intersection between learning and auctions. For example, several recent works have studied second-price auctions from the seller's perspective who aims for an optimal reserve price~\cite{medina2014learning,cesa2014regret,roughgarden2019minimizing,zhao2020online}, although a few papers have taken the bidder's perspective in second-price auctions \cite{mcafee2011design,weed2016online}, where the bidder does not have perfect knowledge of her own valuations. 

However, the problem of learning to bid in repeated first-price auctions has remained largely unexplored, and has only begun to see developments in two recent works \cite{balseiro2019contextual,han2020optimal} which obtained the optimal regrets of repeated first-price auctions under different censored feedback structures. Despite these pioneering efforts, which shed important light in the previously uncharted territory of learning in repeated first-price auctions, much more remains to be done. In particular, a key assumption made in these works is that the highest other bid (abbreviated as the HOB thereafter) follows some \textit{iid} but unknown distributions to make learning possible, while it may not always hold in practice. Conceptually, this is quite ubiquitous and easy to understand: many, if not all, of the other bidders can be using various sophisticated (and unknown) bidding strategies themselves, thereby yielding a quite complicated HOBs over time, which may not be following any distribution at all. Moreover, there may even exist adversarial bidders and sellers who aim to take advantage of the bidder, which makes the HOB depend on the behavior of the given bidder and further non-\emph{iid}. Empirically, this non-\emph{iid} phenomenon can also be observed in various practical data, e.g. the real first-price auction datasets obtained from Verizon Media in Section~\ref{sec:experiment}. Consequently, it remains unknown and unclear how to adaptively bid in this more challenging adversarial first-price auctions, where others' bids, and in particular the HOB, can be arbitrary. 

 
 Consequently, we are naturally led to the following questions: Is it still possible to achieve a sub-linear regret against a rich set of bidding strategies in repeated first-price auctions where others' bids could be adversarial? If the answer is affirmative, how can we design an online learning algorithm that adaptively bids in this setting in order to achieve the optimal performance? We address these questions in depth in this paper.

\subsection{Our Contributions}
Our contributions are threefold. First, we study the problem of learning to adaptively bid in adversarial repeated first-price auctions, and show that an $\widetilde{O}(\sqrt{T})$ regret is achievable when competing with the set of all Lipschitz bidding policies, a strong oracle that contains a rich set of practical bidding strategies (cf. Theorem \ref{thm:main}). We also illustrate the importance of the choice of the competing bidding policies in adversarial first-price auctions. Specifically, when the bidder is competing with the set of all monotone bidding policies, which is another natural set of bidding strategies in practice, a sub-linear regret in the worst case is impossible (cf. Theorem \ref{thm:monotone}). Moreover, we develop the first minimax optimal algorithm, which we call Chained Exponential Weighting (ChEW), to compete with all Lipschitz bidding policies. This algorithm is designed by combining two insights: 
a) in the problem of prediction with expert advice, the presence of a good expert can be leveraged to significantly enhance the regret performance (cf. Theorem~\ref{thm:good_expert}); b) a hierarchical expert-chaining structure can be designed through more and more refined chains that gradually covers a wide enough set of basis Lipschitz bidding policies that essentially replicates the performance of the best Lipschitz policy (cf. Theorem \ref{thm:chew}). These chained experts form a tree, with successive levels of nodes representing larger and larger expert sets. 
The final $\widetilde{O}(\sqrt{T})$ regret performance is achieved both by a careful choice of what each internal node and each leaf node would do respectively on the algorithmic front and by a delicate balance of regret decomposition through these nodes in the hierarchy on the analysis front, leveraging in particular the fact that there exists a good expert among the children of each node. 

Second, although ChEW is near-optimal statistically, it is computationally infeasible as it needs to deal with exponentially many experts in each iteration. To mitigate this drawback, we carefully select an appropriate set of experts to exploit the possible product structure that exists in the problem and provide an important modification of ChEW to form a computationally efficient and space efficient algorithm, which we call Successive Exponential Weighting (SEW). Specifically, the SEW policy enjoys an $O(T)$ space complexity and an $O(T^{3/2})$ time complexity (Theorem~\ref{thm:SEW}), both of which are computationally efficient enough for it to run very quickly in practice. Furthermore, the SEW policy bypasses the ``curse of Lipschitz covering'' suffered by ChEW without paying any price in its statistical guarantee, retaining the same $\widetilde{O}(\sqrt{T})$ minimax optimal regret bound. 

Third, we test our algorithm (the SEW policy) on three real-world first-price auction datasets obtained from Verizon Media, where each of the three datasets has its own individual characteristics (cf. Section~\ref{sec:experiment}). Without the aids of any other side information, we compare the performance of our algorithm to that of three existing bidding algorithms, covering parametric modelings of the optimal bids, as well as the natural learning-based algorithm given the \emph{iid} assumption. Experimental results show that each of the competing bidding algorithms perform well on certain datasets, but poorly on others. However, our algorithm performs robustly and uniformly better over all competing algorithms on all three datasets, thus highlighting its strong empirical performance and practical deployability, in addition to (and also in our view because of) its strong theoretical guarantees.

\subsection{Related Prior Work}
The problem of prediction with expert advice has a long history dating back to repeated games \cite{hannan1957approximation} and individual sequence compression \cite{ziv1978coding,ziv1980distortion}, where the systematic treatment of this paradigm was developed in \cite{vovk1990aggregating,littlestone1994weighted,cesa1997use,vovk1998game}. We refer to the book \cite{cesa2006prediction} for an overview. Both insights used in the ChEW policy on prediction with expert advice have appeared in literature. For the first insight that a good expert helps to reduce the regret, it was known that data-dependent regret bounds depending on the loss of the best expert are available either for constant \cite{littlestone1994weighted,freund1995desicion} or data-dependent learning rates \cite{yaroshinsky2004better}. For the second insight on the expert-chaining structure, the idea of chaining dated back to \cite{dudley1967sizes} to deal with Gaussian processes, and a better result can be obtained via a generic chaining \cite{talagrand2006generic}. As for the applications of the chaining idea in online learning, the log loss was studied in \cite{opper1999worst,cesa2001worst}, and \cite{rakhlin2014online,rakhlin2017empirical} considered the square loss. For general loss functions, the online nonparametric regression problem was studied in \cite{gaillard2015chaining,rakhlin2015online}, and \cite{cesa2017algorithmic} considered general contextual bandits with full or censored feedbacks. However, a key distinguishing factor between the previous works and ours is that the loss (or reward) function in first-price auctions is not continuous, resulting in a potentially large performance gap among the children of a single internal node in the chain and rendering the previous arguments inapplicable. Consequently, we use a combination of the chaining idea and the good expert to arrive at an optimal policy statistically, which is novel to the best of the authors' knowledge. 

To reduce the computational complexity of running the exponential weighting algorithm over a large number of experts, there have been lots of works on efficient tracking of large classes of experts especially when the experts can be switched a limited number of times. A popular approach in the information theory/source coding literature is based on transition diagrams \cite{willems1996coding,shamir1999low}, which are used to define a prior distribution on the switches of the experts. Another method with a lower complexity was proposed in \cite{herbster1998tracking}, which was shown to be equivalent to an appropriate weighting in the full transition program \cite{vovk1999derandomizing}. Another variant, called the reduced transition diagram, has also been used for logarithmic losses \cite{willems1997live,shamir1999low} and general losses \cite{gyorgy2008efficient,hazan2009efficient}. Successful applications of the previous methods include the online linear optimization \cite{hazan2007logarithmic}, lossless data compression \cite{krichevsky1981performance,willems1995context}, the shortest path problem \cite{takimoto2002path,kalai2005efficient}, or limited-delay lossy data compression \cite{linder2001zero,weissman2002limited,gyorgy2004efficient}. We refer to \cite{gyorgy2012efficient} for an overview. However, our setting is fundamentally different from the above works: in the above works the large set of experts was obtained from a relatively small number of base experts, while the set of experts is intrinsically large in first-price auctions. Consequently, instead of applying the popular approach where one constructs a suitable prior to ease the sequential update, we choose an appropriate set of experts and exploit the novel product structure among the chosen experts to arrive at the desired computational efficiency. For comparison with previous works \cite{gaillard2015chaining,cesa2017algorithmic} on efficient chaining algorithms, the method in \cite{gaillard2015chaining} requires convex losses and gradient information neither of which is available in first-price auctions, and \cite{cesa2017algorithmic} used a different discrete structure to achieve an $O(T^{1.55})$ time complexity at each round, which is more expensive than the $O(\sqrt{T})$ complexity achieved by this work. 

We also review and compare our work with the recent works \cite{balseiro2019contextual,han2020optimal} on repeated first-price auctions. Specifically, \cite{balseiro2019contextual} studied the case of binary feedbacks, where the bidder only knows whether she wins the bid or not after each auction. In this setting, \cite{balseiro2019contextual} provided a general algorithm based on cross learning in contextual bandits and achieved the minimax optimal $\widetilde{\Theta}(T^{2/3})$ regret. Subsequently, \cite{han2020optimal}  studied learning in repeated first-price auctions with censored feedback, where the bidder only observes the winning bid price each time (i.e. the price at which the transaction takes place) and cannot observe anything if she wins the bid. In this slightly more informative setting, \cite{han2020optimal} provided two algorithms from different angles to achieve the optimal $\widetilde{\Theta}(T^{1/2})$ regret when the bidder's private values are stochastic and adversarial, respectively. Our work is different from these works in two aspects. First, we assume a full-information feedback model where others' highest bid is always revealed at the end of each auction, which holds in several online platforms that implement first-price auctions (including Google Ad Manager, the largest online auction platform) where every bidder is able to observe the minimum bid needed to win (which is the HOB). Second, and more importantly, we drop the \emph{iid} assumption of others' highest bids assumed in both works and assume a much more general adversarial framework, which makes the distribution estimation idea presented in both works inapplicable. In summary, this work mainly focus on how to bid without any distributional assumption in reality, while the above works were devoted to the optimal use of censored data in first-price auctions, thus the results are not directly comparable. 


\subsection{Notation}
For a positive integer $n$, let $[n]\triangleq \{1,2,\cdots,n\}$. For a real number $x\in\bR$, let $\lfloor x \rfloor$ be the largest integer no greater than $x$. For a square-integrable random variable $X$, let $\bE[X]$ and $\var(X)$ be the expectation and variance of $X$, respectively. For any event $A$, let $\1(A)$ be the indicator function of $A$ which is one if the event $A$ occurs and zero otherwise. For probability measures $P$ and $Q$ defined on the same probability space, let $D_{\text{KL}}(P\|Q) = \int dP\log\frac{dP}{dQ}$ be the Kullback--Leibler (KL) divergence between $P$ and $Q$. We also adopt the standard asymptotic notations: for non-negative sequences $\{a_n\}$ and $\{b_n\}$, we write $a_n = O(b_n)$ if $\limsup_{n\to\infty} a_n/b_n < \infty$, $a_n = \Omega(b_n)$ if $b_n = O(a_n)$, and $a_n = \Theta(b_n)$ if both $a_n = O(b_n)$ and $a_n = \Omega(b_n)$. We also adopt the notations $\widetilde{O}(\cdot), \widetilde{\Omega}(\cdot), \widetilde{\Theta}(\cdot)$ to denote the respective meanings above within multiplicative poly-logarithmic factors in $n$. 

\subsection{Organization}
The rest of the paper is organized as follows. Section \ref{sec:formulation} presents the setting and main results of this paper. In Section \ref{sec:chaining}, we provide a bidding policy based on a chain of exponentially many experts, and show that it achieves the statistically optimal regret bound. In particular, Section \ref{subsec:good_expert} emphasizes the importance of the presence of a good expert on obtaining a small regret. Building on Section \ref{sec:chaining}, Section \ref{sec:computation} presents a modification of the previous bidding policy and shows that a product structure of experts at each layer of the chain leads to computational efficiency. Further discussions are placed in Section \ref{sec:discussion}, where in particular both possibility and impossibility results for monotone oracles are obtained. Experimental results on the real first-price auction data are included in Section \ref{sec:experiment}, and remaining proofs are relegated to the appendix.

\section{Problem Formulation}\label{sec:formulation}

We consider a repeated first-price auction where a single bidder makes repeated bids during a time horizon $T$. At the beginning of each time $t=1,2,\cdots,T$, the bidder sees a particular \emph{good} and receives a private value $v_t\in [0,1]$ for this good. Based on her past observations of other bidders' bids, the bidder bids a price $b_t\in [0,1]$ for this good, and also let $m_t\in [0,1]$ be the HOB, i.e. the maximum bid of all other bidders. The outcome for the bidder depends on the comparison between $b_t$ and $m_t$: if $b_t \ge m_t$, the bidder gets the good and pays her bidding price $b_t$; if $b_t<m_t$, the bidder does not get the good and pays nothing\footnote{By a slight perturbation, we assume without loss of generality that the bids are never equal.}. Consequently, the instantaneous reward (or utility) of the bidder is
\begin{align}\label{eq.instant_reward}
	r(b_t;v_t, m_t) = (v_t - b_t)\1(b_t \ge m_t). 
\end{align}
Note that in the reward function above, the bidder can only choose her bid $b_t$ but not the private value $v_t$ nor others' highest bid $m_t$. In fact, we assume that the variables $v_t$ and $m_t$ can be arbitrarily chosen by any adversary agnostic to the private random seeds used by the possibly randomized strategy of the bidder. This adversarial assumption eliminates the needs of impractical modeling of others' bids or the private value distribution, and takes into account the possibility that other bidders may use an adaptive bidding strategy and adapt their bids based on others' behaviors. For these variables, we assume a full-information feedback where the private value $v_t$ is known to the bidder at the beginning of time $t$, and others' highest bid $m_t$ is revealed to the bidder at the end of time $t$. The first assumption is natural--the bidder typically obtains her private value $v_t$ for the good once she sees it, provided that she has the perfect knowledge of her own preference or utility. The second assumption holds either in open auctions where the auction organizer reveals every bidder's bid, or in many current online platforms (in particular, such as the Google Ad Manager, the largest online auction platform) where each bidder receives the minimum bid needed to win, which in turn is precisely the HOB $m_t$ (possibly plus one cent). 

A bidding policy $\pi$ consists of a sequence of bids $(b_1,\cdots,b_T)\in [0,1]^T$ such that $b_t$ can only depend on the current private value $v_t$, all observable histories $(v_s,b_s,m_s)_{s<t}$, and her private randomness. The bidder aims to devise a bidding policy $\pi$ where its cumulative reward is close to that achieved by an oracle who knows the entire sequence $(v_t, m_t)_{t\in [T]}$ in advance and is constrained to make bids that smoothly depend on $v_t$. 
This performance gap between what the bidder's policy can achieve and what the oracle can achieve is formalized in the notion of the regret of $\pi$:
\begin{align}\label{eq:regret}
	R_T(\pi) \triangleq \sup_{f \in \calF_{\text{Lip}}}\sum_{t=1}^T \left(r(f(v_t);v_t,m_t) - r(b_t; v_t, m_t) \right),
\end{align}
where $\calF_{\text{Lip}}$ is the collection of all $1$-Lipschitz functions $f: [0,1]\to [0,1]$: $|f(v) - f(v')| \le |v-v'|$ for all $v,v'\in [0,1]$. Note that the oracle's strategy set $\calF_{\text{Lip}}$ is fairly rich: the oracle can adapt the bid to his private valuation of the good (hence he does not need to bid the same price throughout the time horizon), subject to a mild and natural constraint that the dependence on the private value is smooth (meaning that the bids should not differ too much if two goods have similar values). Meanwhile, the restrictions on $\calF_{\text{Lip}}$ are also somewhat necessary due to the following reasons. First, if the oracle is allowed to bid any price depending not only on $v_t$ but also the entire sequence $(v_t, m_t)_{t\in [T]}$, then the best bid is always $b_t  = m_t$ for any $t\in [T]$, while the bidder cannot predict an adversarially chosen $m_t$ based on her past observations. Second, if the bid $b_t$ of the oracle, as a function of $v_t$, can depend on $v_t$ in an arbitrary way without any smoothness constraint, then when the private values $(v_t)_{t\in [T]}$ are all distinct (which occurs easily as the interval $[0,1]$ is a continuum) the oracle can again choose $b_t = m_t$, reducing to the first case. Hence, the choice of $\calF_{\text{Lip}}$
as the benchmark class is both rich and meaningful. 

\subsection{Main Results}
Suprisingly, even if the bidder is competing against a strong oracle with a rich set of strategies in \eqref{eq:regret} under adversarially chosen private values and others' bids, an $\widetilde{O}(\sqrt{T})$ regret is still attainable. 

\begin{theorem}\label{thm:main}
	There exists a randomized bidding strategy $\pi$ such that for all sequences $(v_t)_{t\in [T]}$ and $(m_t)_{t\in [T]}$ taking values in $[0,1]$, the following regret bound holds: 
	\begin{align*}
		\bE[R_T(\pi)] \le C\sqrt{T}\log T, 
	\end{align*}
	where the expectation is taken with respect to the bidder's private randomness, and $C>0$ is an absolute constant independent of $T$ and the sequences $(v_t)_{t\in [T]}$, $(m_t)_{t\in [T]}$.
\end{theorem}

Since an $\Omega(\sqrt{T})$ lower bound of the regret is standard (see, e.g. \cite[Appendix B]{han2020optimal}), Theorem \ref{thm:main} shows that the (near-)optimal regret $\widetilde{O}(\sqrt{T})$ is achievable for adversarial first-price auctions. 
However, despite the existence of a near-optimal policy in terms of the regret, Theorem \ref{thm:main} does not discuss the running time of the policy $\pi$; in fact, the first policy we construct for Theorem \ref{thm:main} will provide clear theoretical insights but suffer from a running time super-polynomial in $T$. The next result shows that, by exploiting the structure of the oracle, this policy can be modified to achieve the same rate of regret but enjoying small time/space complexity. 

\begin{theorem}\label{thm:computation}
	There exists a policy $\pi$ satisfying the regret bound in Theorem \ref{thm:main} with space complexity $O(T)$ and time complexity $O(T^{3/2})$. 
\end{theorem}

Note that although the time complexity $O(T^{3/2})$ is not as good as $O(T)$ that the most optimistic bidder may possibly hope for, it is inherently because of a large (quantized) action space. Specifically, the continuous action space (i.e. all possible bids in $[0,1]$) needs to be quantized into $\Omega(\sqrt{T})$ actions to ensure a small approximation error, while each of the action needs to be evaluated at least once for each auction and therefore contributes to the overall $O(\sqrt{T}\cdot T) = O(T^{3/2})$ complexity. 

\subsection{Additional Related Work}
We also remark that since the initial posting of our arXiv version in Jul 2020, there has been several studies on this topic that generalized the scope of our work. We mention some of the works here. \cite{zhang2021meow} adapted our exponential weighting based algorithm to practical bidding systems with improvements on latency. \cite{zhang2022leveraging} considered a practical scenario where additional side information is available for predicting the HOBs, and established regret guarantees which are adaptive to the performance of the above prediction. \cite{badanidiyuru2023learning} studied a linear model of the HOBs with $d$-dimensional contexts and established a regret upper bound of $\widetilde{O}(\sqrt{dT})$ against the optimal context-dependent bids. Based on the paradigm of regret minimization, \cite{wang2023learning} incorporated the budget constraint into the bidding problem, and \cite{aggarwal2024no} additionally incorporated the return on investment (ROI) constraints. \cite{kumar2024strategically} showed that the $\widetilde{O}(\sqrt{T})$ regret can be improved to $O(\log T)$ when the competing bids are stochastic, and analyzed their bidding algorithm in a strategic environment. \cite{cesa2024role} studied different feedback mechanisms, and showed that an $\widetilde{O}(\sqrt{T})$ regret is possible against a fixed bid even if the private values need to be learned from data. \cite{hu2025learning} relaxed the adversarial environment and studied the dynamic regret for non-stationary HOBs with bounded variation. We also refer to a recent survey \cite{aggarwal2024auto} for an overview of recent progresses on automatic bidding in auctions. 

\section{A Hierarchical Chaining of Experts}\label{sec:chaining}
In this section, we construct a bidding policy in repeated first-price auctions and prove Theorem \ref{thm:main}. In particular, we relate this problem to prediction with expert advice. Specifically, in Section \ref{subsec:good_expert}, we show that in learning with expert advice, the presence of a good expert results in a smaller regret compared to the classical setting. Section \ref{subsec:covering} then formulates the repeated first-price auction as the problem of learning with a continuous set of experts, and shows that a simple covering argument only leads to a regret bound of $\widetilde{O}(T^{2/3})$. To overcome this difficulty, Section \ref{subsec:chaining} proposes a hierarchical chaining of experts so that in each local covering, a good expert helps to reduce the regret and drives the final regret to $\widetilde{O}(\sqrt{T})$. 

\subsection{Prediction with Expert Advice in the Presence of a Good Expert}\label{subsec:good_expert}
We first review the classical setting of prediction with expert advice as follows. 
\begin{definition}[Prediction with Expert Advice]\label{def.expert_advice}
	The problem of prediction with expert advice consists of a finite time horizon $T$, a finite set of $K$ experts, and a reward matrix $(r_{t,a})_{t\in [T], a\in [K]}\in [0,1]^{T\times K}$ consisting of instantaneous rewards achieved by each expert at each given time. At each time $t\in [T]$, the learner chooses an expert $a_t\in [K]$ based on the historical rewards of all experts, receives a reward $r_{t,a_t}$ as a result of following the expert $a_t$'s advice, and observes the rewards $(r_{t,a})_{a\in [K]}$ achieved by all experts. The learner's goal is to devise a (possibly randomized) policy $\pi = (a_1,\cdots,a_T)$ such that the cumulative reward achieved by $\pi$ is close to that achieved by the best expert in hindsight, or in other words, to minimize the regret defined as
	\begin{align*}
	R_T(\pi) = \max_{a\in [K]} \sum_{t=1}^T \left(r_{t,a} - \bE[r_{t,a_t}] \right), 
	\end{align*} 
	where the expectation is taken with respect to the randomness used in the randomized policy $\pi$. 
\end{definition}

It is a well-known result that the optimal regret in prediction with expert advice is $\Theta(\sqrt{T\log K})$ and can be achieved using the exponential-weighting algorithm \cite{littlestone1994weighted}, which randomly chooses each expert with probability proportional to the exponentiated cumulative rewards achieved by that expert. Now we assume an additional structure in the above problem, i.e. there exists an expert who is \emph{good}. 

\begin{definition}[Good Expert]\label{def.good_expert}
	Fix any parameter $\Delta\in [0,1]$. In prediction with expert advice, an expert $a_0\in [K]$ is $\Delta$-\emph{good} (or simply \emph{good}) if for all $a\in[K]$ and $t\in[T]$, 
	\begin{align*}
	r_{t,a_0} \ge r_{t,a} - \Delta. 
	\end{align*}
\end{definition}

In other words, a good expert always achieves a near-optimal instantaneous reward at any time, where the sub-optimality gap is at most $\Delta$. The first reason why we consider a good expert is that the special form of the reward function in \eqref{eq.instant_reward} naturally indicates a good expert in first-price auctions. To see this, consider the scenario where the bidder is suggested to bid within a small interval $[b_0, b_0+\Delta]$, which can be a coarse estimate of the region in which the optimal bid lies. Each bid $b$ in this interval can be thought of as an expert who bids $b$ at this time, and the learner aims to find out the best expert in this interval. An interesting observation is that, the expert who always bids $b_0 + \Delta$ is a $\Delta$-good expert, as for any $b\in [b_0, b_0 + \Delta]$ and $v\ge b,m\in [0,1]$ (note that a rational bidder never bids more than her private valuation), it holds that (cf. equation \eqref{eq:right_lipschitz})
\begin{align*}
r(b_0 + \Delta; v, m) \ge r(b; v,m) - \Delta. 
\end{align*}
Hence, even if we do not know which expert is the \emph{optimal} one, we know that there is a \emph{good} expert. Note that here we also know \emph{who} is the good expert, but we will show in this subsection that the identity of the good expert is in fact unnecessary to achieve the better regret. 

%
%

Another reason to consider a good expert is that the existence of a good expert improves on the classical regret $\Theta(\sqrt{T\log K})$, which is the central claim of this subsection. It is straightforward to see that when $\Delta$ is super small (e.g. $\Delta \le T^{-1}$) and the learner \emph{knows} who is the good expert, she may always follow the good expert safely, resulting in a total regret that is at most $T\Delta$. However, this simple scheme breaks down when $\Delta$ becomes large (though still smaller than $1$), say, $\Delta = T^{-0.1}$. Before we propose a new scheme for reasonably large $\Delta$'s, we compare the notion of the $\Delta$-goodness with similar notions in online learning and show that the classical regret cannot be improved under other notions:
\begin{itemize}
	\item First we consider the notion where the goodness is measured \emph{on average}. Specifically, consider a simple scenario where there are only two experts (i.e. $K=2$), and the rewards of both experts are stochastic and \emph{iid} through time. Further assume that the reward distributions of the experts are $\mathsf{Bern}(1/2)$ and $\mathsf{Bern}(1/2 + T^{-1/2})$, respectively, but we do not know which expert has a higher mean reward. Note that \emph{on average}, both experts are $\Delta$-good as the difference in the mean rewards is $T^{-1/2}\le \Delta$. However, standard information-theoretic analysis reveals that any learner will make a constant probability of error in distinguishing between these two scenarios, and whenever an error occurs, the worst-case regret is at least $\Omega(T^{-1/2}\cdot T) = \Omega(\sqrt{T})$, which is independent of $\Delta$ as long as $\Delta\ge T^{-1/2}$. Hence, reduction on the regret is impossible when the goodness is measured through the mean reward, and the fact that the gap \emph{never} exceeds $\Delta$ really matters.
	\item Second we consider another notion where there is a \emph{lower} bound on the suboptimality gap, i.e. when the best expert outperforms any other arms by a margin at least $\Delta$ in expectation. In the literature of stochastic bandits, it is known that better regret bounds such as $O(\Delta^{-1}\log T)$ are possible compared to $O(\sqrt{T})$ \cite{bubeck2012regret}. However, this notion is fundamentally different from our definition of $\Delta$-goodness: a lower bound on the suboptimality gap makes the detection of the best expert significantly easier, while with only an upper bound, it might still be hard to tell whether a good expert is the best expert.
	\item  Third, we remark that correlated rewards among actions are not helpful in most \emph{bandit} problems where only the reward of the chosen action is revealed. Specifically, Definition \ref{def.good_expert} implies a special type of correlation among actions at each time, but with bandit feedbacks where only the reward of one action can be observed at each time, the above correlations are always lost and the observed rewards can be mutually independent through time. Hence, the correlation structure of Definition \ref{def.good_expert} can only be helpful in feedback models stronger than the bandit feedback, such as the full information scenario we are considering. 
\end{itemize}

Despite the above negative results, we show that the answer is always \emph{affirmative} under our new definition of the $\Delta$-goodness. We illustrate the idea using a simple example. 
\begin{example}
Let us revisit the scenario where there are $K=2$ experts with iid rewards over time, and expert $1$ is a $\Delta$-good expert. Without a good expert, the traditional successive elimination policy \cite{even2006action} at time $t$ will eliminate the probably bad expert whose average historical reward is below the other by a margin of $\widetilde{O}(1/\sqrt{t})$ and choose the other one, or choose either one if the experts achieve similar rewards. The rationale behind the choice of the margin is that the difference of rewards, as a random variable, always lies in $[-1,1]$ and is thus sub-Gaussian, and therefore $t$ observations help to estimate its mean within accuracy $\widetilde{O}(1/\sqrt{t})$ with high probability. Hence, we may safely assume that the better expert is never eliminated, and whenever the other expert is still present at time $t$, its suboptimality gap is at most $\widetilde{O}(1/\sqrt{t})$. As a result, the instantaneous regret at time $t$ is either $0$ or $\widetilde{O}(1/\sqrt{t})$, and the total regret is at most $\sum_{t=1}^T \widetilde{O}(1/\sqrt{t}) = \widetilde{O}(\sqrt{T})$. 

How does a good expert help in this scenario? It helps in two aspects. First, we may always choose expert $1$ \emph{by default} unless we are fairly confident that expert $2$ is better. Second, the estimation error of the reward difference will be smaller, \emph{if} expert $2$ is indeed better than expert $1$. In fact, the random variable $\delta_t \triangleq r_{t,2} - r_{t,1}$ will always take value in $[-1,\Delta]$ thanks to the $\Delta$-goodness of expert $1$, thus if $\bE[\delta_t]\ge 0$ (i.e. the distribution of $\delta_t$ leans towards the right end of the interval), we have
\begin{align*}
\var(\delta_t) \le \bE[(\Delta - \delta_t)^2] \le (1+\Delta)\cdot \bE[\Delta - \delta_t] \le (1+\Delta)\Delta \le 2\Delta. 
\end{align*}
Hence, by Bernstein's inequality (cf. Lemma \ref{lemma.bernstein}), we have
\begin{align*}
\bP\left(\left|\frac{1}{t}\sum_{s=1}^t (\delta_s - \bE[\delta_s])\right| \ge \varepsilon \right) \le 2\exp\left(-\frac{t\varepsilon^2}{2(2\Delta + \varepsilon/3)}\right). 
\end{align*}
This implies that the mean difference $\bE[\delta_t]$ can be estimated within accuracy $\widetilde{O}(\sqrt{\Delta/t} + 1/t)$, which is better than $\widetilde{O}(\sqrt{1/t})$ without a good expert. Hence, the learner may always choose expert $1$, and only switches to expert $2$ if its average reward exceeds that of expert $1$ by a margin at least $\widetilde{\Theta}(\sqrt{\Delta/t} + 1/t)$. As a result, if expert $2$ is worse than expert $1$, then it is unlikely that expert $2$ will be chosen and the resulting regret is small. If expert $2$ is better than expert $1$, then the improved estimation error above shows that if expert $1$ is still chosen at time $t$, then the mean reward difference is at most $\widetilde{O}(\sqrt{\Delta/t} + 1/t)$. Therefore, the overall regret is at most $\sum_{t=1}^T \widetilde{O}(\sqrt{\Delta/t} + 1/t) = \widetilde{O}(\sqrt{T\Delta})$. 
\end{example}

Although in the above example the identity of the good expert is known, the next theorem shows that the minimax regret reduces from $\Theta(\sqrt{T\log K})$ to $\Theta(\sqrt{T\Delta\log K})$ in general. 

\begin{theorem}\label{thm:good_expert}
	Let $\Delta\in [T^{-1}(1+\log T)^2\log K, 1]$. Then there exist absolute constants $C>c>0$ such that
	\begin{align*}
	c\sqrt{T\Delta\log K} \le \inf_{\pi} \sup_{(r_{t,a})} R_T(\pi) \le C\sqrt{T\Delta\log K},
	\end{align*}
	where the supremum is taken over all possible rewards $(r_{t,a})_{t\in [T], a\in [K]}\in[0,1]^{T\times K}$ such that there is a $\Delta$-good expert, and the infimum is taken over all possible policies $\pi$. In particular, for any $\Delta>0$ and the classical exponential-weighting policy $\pi^{\text{\rm EW}}$ with a time-varying learning rate $\eta_t = \min\{1/4, \sqrt{(\log K)/(t\Delta)}\}$, the following regret upper bound holds: 
	\begin{align*}
	 R_T(\pi^{\text{\rm EW}}) \le 4\sqrt{T\Delta\log K} + 32(4+\log T)\log K. 
	\end{align*}
\end{theorem}

Theorem \ref{thm:good_expert} implies that as long as there is a good expert with $\Delta=o(1)$, the minimax regret can be improved. Furthermore, this minimax regret is strictly larger than that of the special case in which the rewards of all experts lie in a bounded interval with range $\Delta$, where the minimax regret is $\Theta(\Delta\sqrt{T\log K})$ by simple scaling arguments. Note that when $\Delta\le T^{-1}\log K$, the simple policy where the good expert is always chosen is already optimal, with a total regret $T\Delta$, thus the range of $\Delta$ in Theorem \ref{thm:good_expert} is near-optimal. 

The exponential-weighting policy $\pi^{\text{EW}}$ with a time-varying learning rate is as follows: at time $t\in [T]$, each expert $a\in [K]$ is chosen with probability
\begin{align*}
p_{t,a} = \frac{\exp\left(\eta_{t}\sum_{s<t} r_{s,a} \right) }{\sum_{a'\in [K]} \exp\left(\eta_{t}\sum_{s<t} r_{s,a'}\right)}. 
\end{align*}
In practice, the probability can be updated via the rule $p_{t+1,a} \propto p_{t,a}^{\eta_{t+1}/\eta_{t}}\exp(\eta_{t+1}r_{t,a})$ for all $a\in [K]$. The reason why we choose a time-varying learning rate is to remove the algorithmic dependence on the possibly unknown time horizon, which will be helpful in later algorithms. We remark that a properly chosen constant learning rate $\eta_t \equiv \sqrt{(\log K)/(T\Delta)}$ also works for Theorem \ref{thm:good_expert} (even with the optimal range of $\Delta\in [T^{-1}\log K,1]$), which easily follow from the data-dependent regret analysis of the exponential-weighting algorithm \cite{littlestone1994weighted,freund1995desicion}; see also \cite[Theorem 2.4]{cesa2006prediction}. However, the analysis of the shrinking learning rate is more involved, and we relegate it (as well as the minimax lower bound) to Appendix \ref{appendix:good_expert}.

\subsection{A Continuous Set of Experts}\label{subsec:covering}
To formulate the repeated first-price auctions as prediction with expert advice, we first need to specify the set of experts. However, unlike the classical setting where the learner is competing against an oracle who takes the best \emph{fixed} action, in first-price auctions the oracle's action (i.e. bid) can depend on the private value $v_t$, which can be treated as \emph{contexts}. Due to the existence of the contexts, one cannot relate the first-price auctions directly to learning from experts bidding the same price over time. 

However, the contexts can be handled by considering a larger number, or even a continuous set, of experts. Specifically, for each bidding strategy $f(\cdot)\in\calF_{\text{Lip}}$ (i.e. a $1$-Lipschitz function from $[0,1]$ to $[0,1]$), we may associate it with an expert using this bidding strategy. In other words, we identify the function class $\calF_{\text{Lip}}$ as a continuous set of experts. In this way, the setting of repeated first-price auctions coincides with that of prediction with expert advice: at time $t\in [T]$ with private value $v_t$ and HOB $m_t$, each expert $f\in\calF_{\text{Lip}}$ receives a reward
\begin{align*}
r_{t}(f) = r(f(v_t); v_t, m_t)
\end{align*}
with the reward function $r(b;v,m)$ given in \eqref{eq.instant_reward}. Moreover, since both quantities $v_t$ and $m_t$ can be observed at the end of time $t$, the rewards of all experts can be observed as well. Finally, although the bidder is not restricted to choose from the experts, she is essentially doing so as for any price $b$ she bids under the private value $v$, there exists some $f\in \calF_{\text{Lip}}$ such that $b = f(v)$. In other words, it is equivalent to say that the bidder chooses the expert $f$ at that time. Hence, let $f_t\in\calF_{\text{Lip}}$ be the expert chosen by the bidder at time $t$, the bidder's regret in \eqref{eq:regret} is equivalent to
\begin{align}\label{eq:regret_expert}
R_T(\pi) = \max_{f\in \calF_{\text{Lip}}}\sum_{t=1}^T r_t(f) - \sum_{t=1}^T \bE[r_t(f_t)], 
\end{align}
which is exactly the regret in prediction with expert advice. 

A potential difficulty of the above formulation is that the cardinality of experts is infinite, thus known results for a finite number of experts cannot be directly applied. However, one may suggest that a finite set of experts may be sufficient to approximate, or ``cover'', the entire set $\calF_{\text{Lip}}$, which requires the following suitable notion of continuity among experts. 
\begin{lemma}\label{lemma:continuity}
	Let $\calF_0\subseteq \calF_{\text{\rm Lip}}$ be the subset of experts $f\in\calF_{\text{\rm Lip}}$ such that $f(v)\le v$ for all $v\in [0,1]$. Then if $f,g\in \calF_0$ satisfy $g(v) - \varepsilon\le f(v) \le g(v)$ everywhere on $[0,1]$, it always holds that
	\begin{align*}
	\sum_{t=1}^T r_t(f) \le \sum_{t=1}^T r_t(g) + T\varepsilon. 
	\end{align*}
\end{lemma}
	
Note that the reduction from $\calF_{\text{\rm Lip}}$ to a smaller set $\calF_0$ does not hurt: an expert who bids a higher price than her private valuation can always be improved. Hence, in the sequel we may always work with $\calF_0$ instead of $\calF_{\text{\rm Lip}}$. Lemma \ref{lemma:continuity} shows that, if two experts are close in the $L_\infty([0,1])$-norm with an additional constraint that one expert always bids higher than the other, then the total reward achieved by the expert with higher bids is essentially as good as the other one. In other words, if we could find a finite set of experts $\calN_\varepsilon\subseteq \calF_0$ such that for every $f\in\calF_0$, there exists some function $g\in \calN_\varepsilon$ such that $g-\varepsilon\le f\le g$, then
\begin{align*}
\max_{f\in \calF_{\text{\rm Lip}}} \sum_{t=1}^T r_t(f) \le \max_{g\in \calN_\varepsilon} \sum_{t=1}^T r_t(g) + T\varepsilon. 
\end{align*} 
The minimum cardinality of $\calN_\varepsilon$ is known as the \emph{$\varepsilon$-bracketing number} of $\calF_0$, which is characterized in the following lemma. 

\begin{lemma}[\!\!\cite{kolmogorov1959varepsilon}]\label{lemma:covering}
	There exists an absolute constant $C_{\text{\rm Lip}}>0$ such that for any $\varepsilon\in (0,1)$, there is a finite set $\calN_\varepsilon = \{g_1,\cdots,g_{N_\varepsilon}\} \subseteq \calF_0$ with $\calF_0\subseteq \cup_{i=1}^{N_\varepsilon} (g_i - \varepsilon, g_i)$ and $N_\varepsilon = |\calN_\varepsilon| \le \exp(C_{\text{\rm Lip}}/\varepsilon)$. 
\end{lemma}
\begin{remark}
	Although \cite{kolmogorov1959varepsilon} only established the $\varepsilon$-covering number of general H\"{o}lder classes, the extension to $\varepsilon$-bracketing numbers is straightforward, see, e.g. \cite[Corollary 2.7.2]{van1996weak}. 
\end{remark}

Since the set of experts is now finite, standard strategies for prediction with expert advice can be employed to obtain a finite regret in \eqref{eq:regret_expert} against any expert in $\calN_\varepsilon$. Consequently, recall that the minimax regret of learning from $K$ experts is $\Theta(\sqrt{T\log K})$, we have
\begin{align*}
R_T(\pi) &=  \max_{f\in \calF_{\text{\rm Lip}}}\sum_{t=1}^T r_t(f) - \sum_{t=1}^T \bE[r_t(f_t)] \\
&\le T\varepsilon + \max_{f\in \calN_\varepsilon}\sum_{t=1}^T r_t(f) - \sum_{t=1}^T \bE[r_t(f_t)] \\
&\le T\varepsilon + C\sqrt{T\log |\calN_\varepsilon|} \\
&\le T\varepsilon + C\sqrt{\frac{C_{\text{\rm Lip}}T}{\varepsilon}},
\end{align*}
where $C,C_{\text{\rm Lip}}>0$ are absolute constants appearing in Theorem \ref{thm:good_expert} and Lemma \ref{lemma:covering}, respectively, and the last inequality is due to the cardinality bound in Lemma \ref{lemma:covering}. Consequently, there is a trade-off between the approximation error (which becomes smaller as $\varepsilon$ decreases) and the regret against expert set $\calF_0$ (which becomes smaller as $|\calF_\varepsilon|$ decreases, or equivalently, $\varepsilon$ increases), and the optimal choice of $\varepsilon = \Theta(T^{-1/3})$ gives an $O(T^{2/3})$ regret. However, this regret is larger than $\widetilde{O}(\sqrt{T})$, showing that a simple bracketing argument is not sufficient. In the next subsection, we will further make use of the continuity among experts in $\calN_\varepsilon$ to effectively reduce the cardinality of experts. 

\subsection{A Hierarchical Chaining}\label{subsec:chaining}
In this section, we generalize the previous bracketing arguments to a hierarchical chaining of brackets, where the continuity structure of experts is used not only in the final approximation step but in all intermediate steps. The high-level idea is similar to \cite{cesa2017algorithmic} (as well as more related works in the introduction), while some modifcations are necessary to account for good experts. 

Roughly speaking, instead of considering a single $\varepsilon$-bracket of $\calF_0$, we construct $M = \lfloor \log_2 \sqrt{T}\rfloor$ different brackets with different approximation levels $\varepsilon_1>\varepsilon_2 >\cdots>\varepsilon_M$ in a hierarchical manner. Specifically, for each level $m\in [M]$, let $\calN_m$ be an $\varepsilon_m$-bracket of $\calF_0$ as in Lemma \ref{lemma:covering}, with cardinality $N_m \triangleq |\calN_m| \le \exp(C_{\text{\rm Lip}}/\varepsilon_m)$. We also adopt the convention that $\calN_0 = \{f_0\}, N_0 =1$ and $\varepsilon_0=1$ for $m=0$, where $f_0$ is any function in $\calF_0$. These levels form a hierarchical tree structure where the bidder (represented by the only element $f_0\in \calN_0$) is the root, and each expert $f\in \calN_M$ has a chain of ``managers'' $f_{M-1}\to f_{M-2}\to \cdots \to f_1\to f_0$ where for $m=0,1,\cdots,M-1$, 
\begin{align}\label{eq:manager}
f_m = \arg\min_{f\in \calN_m} \|f - f_{m+1}\|_\infty, 
\end{align}
with ties broken arbitrarily. In other words, each function $f_{m+1}\in \calN_{m+1}$ picks the closest function in the $m$-th bracket $\calN_m$ as its \emph{manager}, and $f_{m+1}$ is called an \emph{employee} of $f_m$. We also call $f_m$ as the manager of the expert $f\in \calN_M$ at level $m$, and by convention we have $f_M = f$, and all elements of $\cup_{m=0}^{M-1} \calN_m$ are ``managers'' (internal nodes of the tree). Similarly, all elements of $\calN_M$ are called ``experts'' (leaf nodes of the tree). In the sequel, we abuse the notation slightly and use $f_m$ to denote either the manager of a given expert $f$ at level $m$ or a generic element of $\calN_m$, depending on the context. 

As before, the experts $f \in \calN_M$ always use the bidding strategy $f$ and receive a reward $r_t(f)$ at each time $t$. However, the bidding strategy of any manager $f_m$ is no longer the strategy $f_m$ itself, but a probability distribution $P_{t,f_m}$ over the experts that $f_m$ manages. In other words, at each time $t$, the manager $f_m$ chooses to follow a random expert in $\calN_M$ according to the probability distribution $P_{t,f_m}$, which may depend on all observations up to time $t-1$. Hence, if we define $P_{t,f}$ as the Dirac delta measure on $\{f\}$ for all $f\in \calN_M$, the expected reward difference between the bidder and any expert $f\in \calF_M$ at time $t$ can be expressed as 
\begin{align*}
r_t(f) - \bE_{f_t\sim P_{t,f_0}}[r_t(f_t)] = \sum_{m=1}^M \left( \bE_{f_t\sim P_{t,f_m}}[r_t(f_t)] - \bE_{f_t\sim P_{t,f_{m-1}}}[r_t(f_t)] \right). 
\end{align*}
Equivalently, the expected reward difference between the bidder and a fixed expert $f$ is the sum of the reward differences between the adjacent managers of $f$ over different levels. Hence, the bidder's regret in \eqref{eq:regret_expert} can be upper bounded as
\begin{align}\label{eq:regret_decomposition}
R_T(\pi) \le T\varepsilon_M + \sum_{m=1}^M \max_{ \substack{f_{m-1}\in \calN_{m-1}, f_m \in \calN_m, \\  f_{m} \to f_{m-1}} } \sum_{t=1}^T\left( \bE_{f_t\sim P_{t,f_m}}[r_t(f_t)] - \bE_{f_t\sim P_{t,f_{m-1}}}[r_t(f_t)] \right),
\end{align}
where $f_m\to f_{m-1}$ indicates that $f_{m-1}$ is the manager of $f_m$ in the sense of \eqref{eq:manager}. As a result, the inequality \eqref{eq:regret_decomposition} shows that the overall regret will be small if the expected reward difference between any adjacent managers is small. In other words, all managers are under the setting of prediction with expert advice and aim to achieve a small regret against their best next-level employees. 

However, the na\"{i}ve application of the regret bounds for prediction with expert advice to \eqref{eq:regret_decomposition} will still break down. This is because the number of employees of a manager $f_{M-1}\in \calN_{M-1}$ can still be as large as $\exp(\Omega(1/\varepsilon_M))$, the cardinality of $\calN_M$, and therefore the regret of the manager $f_{M-1}$ can be as large as $O(\sqrt{T/\varepsilon_M})$, still leading to an $O(T^{2/3})$ regret bound. Hence, the additional structure that each manager has similar employees needs to be used. In fact, for each $f_{m+1}\in \calN_{m+1}$ with the manager $f_m \in \calN_m$, the $\varepsilon_m$-bracketing property of the set $\calN_m$ implies that $\min_{f\in \calN_m}\|f_{m+1} - f\|_\infty \le \varepsilon_m$. Hence, by \eqref{eq:manager}, we also have $\|f_{m+1} - f_m\|_\infty \le \varepsilon_m$. As a result, by the triangle inequality, if $f, g \in \calN_M$ are two experts with the same manager $f_m=g_m\in\calN_m$ at level $m$, then 
\begin{align*}
\|f-g\|_\infty \le \sum_{r=m}^{M-1} \left( \|f_{r+1} - f_{r}\|_\infty + \|g_{r+1} - g_{r}\|_\infty \right) \le 2\sum_{r=m}^{M-1} \varepsilon_r. 
\end{align*}
In other words, the support of the distribution $P_{t,f_m}$ lies in a small $L_\infty$-ball with radius depending only on the level $m$. An important implication of the above observation is that there exists a \emph{good} expert compared with all employees of $f_m$. Specifically, for each $f_m\in\calN_m$, define the \emph{dummy expert} $f_m^\star$ with
\begin{align}\label{eq:good_expert}
f_m^\star(v) = \max_{f\in \calN_M: f\to f_m} f(v), \quad v\in [0,1], 
\end{align}
where $f\to f_m$ indicates that $f_m$ is the manager of expert $f$ at level $m$. Since the constraint $f(v)\le v$ and the $1$-Lipschitzness are closed under pointwise maximum, we have $f_m^\star\in\calF_0$. Moreover, for all experts $f$ in the support of $P_{t,f_m}$, it always holds that $f_m^\star - 2\sum_{r=m}^{M-1}\varepsilon_r\le f\le f_m^\star$, indicating that the dummy expert $f_m^\star$ is $\Delta_m$-good compared with all employees of the manager $f_m$ thanks to Lemma \ref{lemma:continuity}, with $\Delta_m = 2\sum_{r=m}^{M-1} \varepsilon_r$. Hence, for each manager $f_m$ we may include $f_m^\star$ explicitly as an expert and run the exponential weighting algorithm with a proper learning rate, which is expected to have a smaller regret due to the presence of a good expert by Theorem \ref{thm:good_expert}. 

\begin{algorithm}[!t]
	\caption{Chained Exponential Weighting (ChEW) Policy	\label{algo:chew}}
	\textbf{Input:} Time horizon $T$; number of levels $M = \lfloor \log_2\sqrt{T} \rfloor$; radii of brackets $\varepsilon_m = 2^{-m}$ and suboptimality gaps $\Delta_{m} =2 \sum_{r=m}^{M-1}\varepsilon_r$ for $m=0,1,2,\cdots,M$. \\
	\textbf{Output:} A bidding policy $\pi$. \\
	\textbf{Initialization:} Find an $\varepsilon_m$-bracket $\calN_m$ of $\calF_0$ with cardinality at most $\exp(C/\varepsilon_m)$ for $m\in [M]$ (existence is ensured by Lemma \ref{lemma:covering}); \\
	\For{$m = M-1, M-2, \cdots, 0$}{
		\For{$f_m \in \calN_m$}{
			Construct the dummy expert $f_m^\star$ in \eqref{eq:good_expert} and add it to $\calN_{M}$; \\
			Set $C(f_m)\gets \{f_m^\star\}\cup\{f_{m+1}\in \calN_{m+1}: f_{m+1} \to f_m\}$;\\
			Initialize $r_{0,f_m}(f) \gets 0$ for all $f\in C(f_m)$. 
		}
	}
	\For{$t=1,2,\cdots,T$}{
		The bidder receives the private value $v_t\in [0,1]$; \\
		Set $P_{t,f}(f') \gets \1(f=f')$ for all $(f,f')\in \calN_M\times \calN_M$; \\
		Set learning rate $\eta_{t,m} \gets \min\{1/4, \sqrt{C_{\text{\rm Lip}}/(t\Delta_{m}\varepsilon_{m+1})}\}$; \\
		\For{$m=M-1,M-2,\cdots,0$}{
			\For{$f_m\in \calN_m$}{
				For each $f_{m+1}\in C(f_m)$, set \begin{align}\label{eq:weights_manager}
				Q_{t,f_m}(f_{m+1}) \gets \frac{\exp\left(\eta_{t,m} r_{t-1,f_m}(f_{m+1})\right)}{\sum_{f\in C(f_m)}\exp\left(\eta_{t,m} r_{t-1,f_m}(f)\right)}. 
				\end{align}\\
				For each $f\in \calN_M$, set
				\begin{align}\label{eq:weights_expert}
				P_{t,f_m}(f) \gets \sum_{f_{m+1}\in C(f_m)} Q_{t,f_m}(f_{m+1})P_{t,f_{m+1}}(f). 
				\end{align}
			}
		}
		The bidder samples $f\sim P_{t,f_0}$ and bids $b_t = f(v_t)$; \\
		The bidder receives the HOB $m_t$; \\
		\For{$m=M-1, M-2, \cdots, 0$}{
			\For{$f_m\in \calN_m$}{
				For each $f_{m+1}\in C(f_m)$, update
				\begin{align}\label{eq:reward_update}
				r_{t,f_m}(f_{m+1}) \gets r_{t-1,f_m}(f_{m+1}) + \sum_{f\in \calN_M} P_{t,f_{m+1}}(f)\cdot r(f(v_t);v_t,m_t). 
				\end{align}
			}
		}
	}
\end{algorithm}

The detailed description of the resulting policy, called ChEW (Chained Exponential Weighting), is displayed in Algorithm \ref{algo:chew}. For the initialization of the algorithm, we set $\varepsilon_m = 2^{-m}$ and construct the hierarchical chaining of experts as above, with a dummy expert $f_m^\star$ added to each manager $f_m$. Hence, the set of employees of a manager $f_m$ consists of all individuals (managers or experts) at level $m+1$ with $f_m$ being their manager plus the dummy expert $f_m^\star$, and is denoted by $C(f_m)$ in Algorithm \ref{algo:chew}. Moreover, for each employee $f$ of $f_m$, we use $r_{t,f_m}(f)$ to store the cumulative past rewards of $f$ up to time $t$, which is initialized to be zero. We also set the time-varying learning rate $\eta_{t,m}=\min\{1/4, \sqrt{C_{\text{\rm Lip}}/(t\Delta_{m}\varepsilon_{m+1})}\}$ for all managers at level $m$, where $C_{\text{\rm Lip}}$ is the absolute constant appearing in Lemma \ref{lemma:covering}. This choice is inspired by Theorem \ref{thm:good_expert}, as $\Delta_m$ is the suboptimality gap of the good expert, and $C_{\text{Lip}}/\varepsilon_{m+1}$ is an upper bound of the log-cardinality $\log|C(f_m)|$ for all $f_m\in \calN_m$. 

Now at each time $t\in [T]$, the ChEW policy proceeds as follows. First, we specify the randomized policy used by all experts and managers in a bottom-up manner. For experts $f\in \calN_M$ (including dummy experts), the bidding policy is simply to use $f$. For any manager $f_m$, her bidding policy is a proper mixture of those of her employees, where the mixture distribution $Q_{t,f_m}$ over $C(f_m)$ is obtained from the exponential weighting algorithm used in Theorem \ref{thm:good_expert}, i.e. the probability that an employee $f_{m+1}\in C(f_m)$ is selected is proportional to her exponentiated past rewards, as shown in \eqref{eq:weights_manager}. Next, since the bidding policy used by any exployee $f_{m+1}\in C(f_m)$ is a known mixture distribution $P_{t,f_{m+1}}$ over all possible experts $f\in \calN_M$, the mixture distribution $Q_{t,f_m}$ over employees naturally induces a mixture distribution $P_{t,f_m}$ over experts for the manager $f_m$ as well, as shown in \eqref{eq:weights_expert}. Now applying the previous procedure from the experts (leaves) to the bidder (root), we obtain the final mixture distribution $P_{t,f_0}$ used by the bidder $f_0$, and the final policy is to randomly pick an expert $f$ from the mixture distribution $P_{t,f_0}$. Finally, since we assume a full-information feedback where both the bidder's private value $v_t$ and HOBs $m_t$ can be observed, the rewards of all experts (and consequently all managers who use a mixture distribution over experts) at time $t$ can be observed. Hence, the cumulative rewards $r_{t,f_m}(f)$ of all employees $f$ of all managers $f_m$ can be updated as shown in \eqref{eq:reward_update}, and we move to time $t+1$. 

The next theorem shows that the ChEW policy achieves an $\widetilde{O}(\sqrt{T})$ regret for repeated first-price auctions, completing the proof of our main Theorem \ref{thm:main}. 

\begin{theorem}\label{thm:chew}
	For any time horizon $T\ge 1$, the \emph{ChEW} policy displayed in Algorithm \ref{algo:chew} satisfies the regret bound 
	\begin{align*}
	R_T(\pi^{\text{\rm ChEW}}) \le \left(2 + C\sqrt{2C_{\text{\rm Lip}} }\log T + 2CC_{\text{\rm Lip}}(1+\log T)\right)\sqrt{T}, 
	\end{align*}
	where the absolute constants $C, C_{\text{\rm Lip}}$ appear in Theorem \ref{thm:good_expert} and Lemma \ref{lemma:covering}, respectively. 
\end{theorem}

The main ideas of the proof of Theorem \ref{thm:chew} are sketched at the beginning of this subsection, and we defer the full proof to the appendix. 
\section{An Efficiently Computable Policy}\label{sec:computation}
Although the ChEW policy achieves the desired $\widetilde{O}(\sqrt{T})$ regret in first-price auctions, an important practical concern is the running time. Specifically, even if one only looks at the bottom level of the chain, there are $\exp(\Omega(1/\varepsilon_M)) = \exp(\Omega(\sqrt{T}))$ experts in total over which the exponential weighting is performed, giving both space complexity and time complexity growing super-polynomially in $T$. In this section, we overcome the computational burden by exploiting the possible product structure among the experts, and propose an efficiently-computable Successive Exponential Weighting (SEW) policy with the same regret guarantee. Specifically, Section \ref{subsec:prod_structure} presents the basic idea of using the product structure to help reduce computation via a simple example, and Section \ref{subsec:sew} details the SEW policy which uses a product structure everywhere in the chain. The analysis of the regret and the space/time complexity of the SEW policy is placed in the appendix. 

\subsection{Importance of Product Structure}\label{subsec:prod_structure}
To reduce the computational complexity of running the exponential weighting algorithm over a large number of experts, there has been a lot of work on efficient tracking of large classes of experts. We refer to the related works in the introduction. However, our problem is fundamentally different from those in the above works in both the class of experts and the methodology:
\begin{enumerate}
	\item The previous works mostly focused on a large class of experts (also called the \emph{meta experts}) obtained from a relatively small number of \emph{base experts}, e.g. meta experts are formed via a limited number of switchings or a convex combination of base experts. In comparison, the number of Lipschitz experts in first-price auctions is \emph{intrinsically} large given by the oracle class, and there is no natural and small class of base experts. 
	\item The previous works typically assumed a fixed class of experts as specified by the problem setup. In contrast, the choice of the finite set of candidate experts can vary in first-price auctions, as different coverings of $\calF_{\text{\rm Lip}}$ will give rise to different sets of experts, and the bidder has the full freedom to choose from different coverings. 
	\item In terms of the methodology, the previous works mostly constructed a special prior distribution on the experts (instead of a uniform distribution) so that the probability update rule in the exponential weighting can be written as a simple recursive form. In contrast, we will still stick to the uniform distribution but choose a structured set of experts instead to ease computation. 
\end{enumerate}

To see how a structured set of experts helps in reducing the computation complexity, consider the policy in Section \ref{subsec:covering} based on a one-stage covering of $\calF_{\text{Lip}}$. Without any structural assumption on the experts, there will be $\exp(\Omega(T^{1/3}))$ experts in total with a prohibitive space and time complexity to run exponential weighting. However, a simple fix is possible by considering the following experts with a product structure: 
\begin{align}\label{eq:product_expert}
\calP = \left\{f: f(v) = \sum_{i=1}^M b_i\1\left(\frac{i-1}{M} < v\le \frac{i}{M} \right) \right\}, 
\end{align}
where $M = T^{1/3}$ and $b_1,\cdots,b_M \in \calB \triangleq \{\frac{1}{M}, \frac{2}{M}, \cdots, 1 \} $. Equivalently, the set $\calP$ consists of all piecewise constant functions on $M$ equally spaced bins with discrete values in the finite action set $\calB$. The set $\calP$ of experts has a product structure because to specify any element $f\in \calP$, it suffices to assign a separate value in $\calB$ to each of the $M$ pieces. Consequently, the overall cardinality of $\calP$ is $|\calP| = M^M = \exp(O(T^{1/3}\log T))$, which is slightly larger than the bracketing number in Section \ref{subsec:covering}. Now the overall policy is to run the vanilla exponential weighting algorithm to the finite class $\calP$ of experts. 

Next we show that the simple modification still achieves an $O(T^{2/3}\sqrt{\log T})$ regret while reduces the computational complexity to $O(T^{1/3})$ per round. For the first claim, note that any $f\in \calF_{\text{Lip}}$ has a good approximation in $\calP$: in fact, the function
\begin{align*}
\widetilde{f}(v) =  \min\left\{b\in \calB: b \ge \sup_{(i-1)/M < v\le i/M} f(v) \right\}
\end{align*}
for $\frac{i-1}{M} < v\le \frac{i}{M}$ satisfies $\widetilde{f}\in \calP$, and $\widetilde{f}-2/M \le f\le \widetilde{f}$ everywhere by the $1$-Lipschitz property of $f$. Hence, Lemma \ref{lemma:continuity} shows that restricting to the experts in $\calP$ only incurs an approximation error of $O(T/M) = O(T^{2/3})$, whereas the bidder's regret compared to the best expert in $\calP$ is $O(\sqrt{T\log |\calP|}) = O(T^{2/3}\sqrt{\log T})$. Hence the claimed regret bound is established for the new algorithm.

For the computational complexity, we claim that the new algorithm is equivalent to running a separate exponential weighting on the actions $\calB$ under each bin. To see this, note that the product structure of $\calP$ implies that the sum $\sum_{f\in\calP}$ can be effectively written as $\sum_{(b_1,\cdots,b_M)\in \calB^M}$ where $f$ is represented by a vector $(b_1,\cdots,b_M)$ indicating the bids on each piece. Hence, let $i(v) = \lceil Mv\rceil\in [M]$ be the index of the bin to which the value $v\in (0,1]$ belongs, in the exponential weighting algorithm over the expert set $\calP$, the probability that the bidder bids price $b\in \calB$ at time $t$ is (we abbreviate $r_t(b) = r(b;v_t,m_t)$, with $r(\cdot)$ the instantaneous reward function defined in \eqref{eq.instant_reward})
\begin{align*}
p_{t}(b)& = \sum_{f\in \calP: f(v_t) = b} p_t(f) \\
&= \frac{\sum_{f\in \calP: f(v_t) = b} \exp(\eta_t\sum_{s<t} r_s(f(v_s)) )}{\sum_{f\in \calP} \exp(\eta_t\sum_{s<t} r_s(f(v_s)) )}\\
&= \frac{\prod_{i\in [M]: i\neq i(v_t)}\sum_{b'\in \calB} \exp(\eta_t\sum_{s<t: i(v_s)=i} r_s(b') )}{ \prod_{i\in [M]}\sum_{b'\in \calB} \exp(\eta_t\sum_{s<t: i(v_s)=i} r_s(b') ) }  \cdot \exp\left(\eta_t\sum_{s<t: i(v_s) = i(v_t)} r_s(b) \right) \\
&= \frac{ \exp(\eta_t\sum_{s<t: i(v_s)=i(v_t)} r_s(b) )}{\sum_{b'\in\calB} \exp(\eta_t\sum_{s<t: i(v_s)=i(v_t)} r_s(b') )}. 
\end{align*}
For a fixed learning rate, it is clear that the last probability corresponds to running an exponential weighting algorithm on the historic data $\{(v_s,m_s)\}_{s<t: i(v_s) = i(v_t)}$ with private valuations falling into the same bin as $v_t$, establishing the claimed equivalence\footnote{Given the equivalence, there is an alternative way to prove the claimed regret bound. Specifically, the cumulative regret incurred at the $i$-th bin is $O(\sqrt{T_i\log |\calB|})$, where $T_i$ is the number of times that $v_t$ falls into the $i$-th bin. Since $\sum_{i=1}^M T_i = T$, the total regret is $O(\sum_{i=1}^M \sqrt{T_i\log|\calB|}) = O(\sqrt{MT\log|\calB|}) = O(T^{2/3}\sqrt{\log T})$.}. For general time-varying learning rates, there are small differences on the actual learning rate used in these two approaches at each time, but both the analysis and the regret bounds are similar. Hence, from the algorithmic perspective, the bidder only needs to maintain a probability vector for each bin, and at each time updates the vector for the only bin where the current value $v_t$ belongs. Consequently, the overall algorithm takes $O(M^2) = O(T^{2/3})$ space and $O(MT) = O(T^{4/3})$ time. 

Let us make two observations from the above example. The first observation is that, whenever the class of experts enjoys a product structure, the overall exponential weighting algorithm on the large class of experts reduces to independent exponential weighting algorithms on multiple bins and can thus be implemented efficiently. Hence, it is expected that if the employees of each manager have a product structure in the hierarchical chain in Section \ref{subsec:chaining}, the overall bidding policy will be efficiently computable. The second observation is that the experts in $\calP$ are not Lipschitz, meaning that they are not legitimate strategies of the oracle. In addition, there are some experts in $\calP$ who are far away from any reasonable (Lipschitz) bidding strategies, e.g. an expert who bids $0$ when $v \le 1/M$ and bids $1$ otherwise. However, what really matters is the approximation property of $\calP$ that every feasible strategy of the oracle can be well approximated by an expert in $\calP$. The piecewise constant construction then shows that, forcing the Lipschitzness \emph{within} each bin helps ensure a good approximation property which is further determined by the number of bins, while no further constraint \emph{across} different bins helps maintain a product structure of experts. In the next subsection, we will show that managers of different levels have different numbers of bins to capture the Lipschitz constraint gradually. 

\subsection{The SEW Policy}\label{subsec:sew}
\begin{algorithm}[!htbp]
	\caption{Successive Exponential Weighting (SEW) Policy	\label{algo:sew}}
	\textbf{Input:} Time horizon $T$; number of levels $L = \lfloor \log_2\sqrt{T} \rfloor$. \\
	\textbf{Output:} A bidding policy $\pi$. \\
	\textbf{Initialization:} Set $M_{\ell} = 2^{\ell+1}, U_{\ell}=2^{\ell+1}-1, W_{\ell} = 2^{\ell}-1$ for $\ell\in [L]$, $I_{\ell,m} = (m-1,m]/M_{\ell}$ for all $\ell\in [L]$ and $m\in [M_{\ell}]$, $b_{\ell,w} = 2^{-\ell}(w+1)$ for all $\ell\in [L]$ and $w\in [W_{\ell}]$. \\
	\For{$\ell\in [L], m\in [M_{\ell}]$}{
		Initialize the visiting time $T_{\ell,m}\gets 0$; \\
		For $u\in [U_{\ell}]$ and $w\in [W_{\ell}]$, initialize cumulative rewards $R_{\ell,m,u} \gets 0, R_{\ell,m,w}'\gets 0$; \\
	}
	\For{$t = 1,2,\cdots,T$}{
		The bidder receives the private value $v_t\in (0,1]$;  \Comment{\emph{\red Step 1. Compute the exponential weights.}}\\
		\For{$\ell\in [L]$}{ 
			
			The bidder identifies $m_{\ell}^\star\in [M_{\ell}]$ with $v_t\in I_{\ell,m_{\ell}^\star}$, and updates $T_{\ell, m_{\ell}^\star} \gets T_{\ell, m_{\ell}^\star}+1$; \\
			For all $w\in [W_{\ell}]$, the bidder computes $R_{\ell,w} \gets (R_{\ell, m_{\ell}^\star, 2w-1}, R_{\ell, m_{\ell}^\star, 2w},R_{\ell, m_{\ell}^\star, 2w+1}, R_{\ell,m_{\ell}^\star,w}')$ and the probability vector 
			\begin{align}\label{eq:EW_prob}
			p_{\ell,m_{\ell}^\star,w} = \mathsf{EW}( R_{\ell,w}, T_{\ell, m_{\ell}^\star}, 2^{1-\ell}) \in \bR_+^4.
			\end{align}
		}
		Initialize $w^\star \gets 1$; \Comment{\emph{\red Step 2. Random action based on the exponential weights.}}\\
		\For{$\ell=1,2,\cdots,L$}{
				
			The bidder draws a random variable $s\in \{1,2,3,4\}$ from distribution $p_{\ell,m_{\ell}^\star,w^\star}$; \\
			\uIf{$s=4$}{The bidder bids $b_t \gets b_{\ell,w^\star}$ and {\bf break};}
			\uElseIf{$\ell<L$}{Update $w^\star \gets 2(w^\star-1)+s$ and {\bf continue};}
			\Else{The bidder bids $b_t\gets 2^{-L-1}(2(w^\star-1)+s)$ and {\bf break}.}
		}
		
		The bidder receives the HOB $m_t$; \Comment{\emph{\red Step 3. Update the rewards.}}\\
		\For{$\ell = L, L-1, \cdots, 1$ and $u\in [U_{\ell}], w\in [W_{\ell}]$}{
			Compute $r_{\ell, w}' \gets r(b_{\ell,w}; v_t, m_t)$ and update $R_{\ell,m_{\ell}^\star, w}' \gets R_{\ell,m_{\ell}^\star, w}' + r_{\ell,w}'$; \\
			\uIf{$\ell = L$}{Compute $r_{\ell, u} \gets r(2^{-L-1}u; v_t, m_t)$ and update $R_{\ell,m_{\ell}^\star, u} \gets R_{\ell,m_{\ell}^\star, u} + r_{\ell, u}$.}
			\Else{Compute \begin{align}\label{eq:update_reward}
				r_{\ell, u} \gets \sum_{s=1}^3 p_{\ell+1,m_{\ell+1}^\star,u}(s)\cdot r_{l+1, 2(u-1)+s} + p_{\ell+1,m_{\ell+1}^\star,u}(4)\cdot r_{l+1,u}'. 
				\end{align}
			}
			Update $R_{\ell, m_{\ell}^\star, u} \gets R_{\ell, m_{\ell}^\star, u} + r_{\ell, u}$. 
		}
	}
\end{algorithm}

\begin{algorithm}[t]
	\caption{Exponential Weighting (EW) \label{algo:ew}}
	\textbf{Input:} Reward vector $(R_1,R_2,R_3,R_4)$; visiting time $t\in \mathbb{N}$; suboptimality gap $\Delta>0$. \\
	\textbf{Output:} A probability vector $(p_1,p_2,p_3,p_4)$. \\
	Set learning rate $\eta = \min\{1/4, \sqrt{(\log 4)/(t\Delta)} \}$; \\
	\For{$i=1,2,3,4$}{Compute the probability
		$$
		p_i = \frac{\exp\left(\eta R_i \right)}{\sum_{j=1}^4 \exp\left(\eta R_j\right)}. 
		$$}
\end{algorithm}

Motivated by the above insights, we propose the Successive Exponential Weighting (SEW) algorithm in Algorithm \ref{algo:sew} taking the classical Exponential Weighting (EW) algorithm with shrinking learning rates (cf. Algorithm \ref{algo:ew}) as a subroutine. A high-level description of the SEW policy is to divide the algorithm into $L = \lfloor \log_2 \sqrt{T}\rfloor$ levels, where different levels have a different number of bins and different classes of experts. Moreover, similar to the ChEW policy in Section \ref{subsec:chaining}, experts at any level (except for the last one) use a mixed bidding strategy and randomly sample from their employees. As will be shown in the detailed description below, an important property is the product structure at each level. 

Before showing how to run Algorithm \ref{algo:sew} sequentially, we first describe the set of experts (or managers/employees, which we may interchangably use depending on the context) at each level. At level $\ell\in [L]$, the continuous set $(0,1]$ of private values\footnote{For simplicity we assume that the private value is never zero, as bidding zero will be clearly optimal in that case.} is partitioned into $M_{\ell} = 2^{\ell+1}$ equally spaced bins $(I_{\ell,m})_{m\in [M_{\ell}]}$. Given each level $\ell$ and bin $m$, there are $U_{\ell} = 2^{\ell+1}-1$ regular experts, indexed by $(\ell, m, u)$ with $u\in [U_{\ell}]$, who will play mixed strategies supported on the interval 
\begin{align}\label{eq:interval_J}
J_{\ell, u} \triangleq \left( 2^{-\ell-1}(u-1), 2^{-\ell-1}(u+1) \right]. 
\end{align}
In addition, there are also $W_{\ell} = 2^{\ell}-1$ \emph{dummy} experts, indexed by $(\ell, m, w)$ with $w\in [W_{\ell}]$, who always bid a fixed price $b_{\ell,w} = 2^{-\ell}(w+1)$ inside the current bin. Then it remains to specify the mixture strategies used by the regular experts, or in particular, the structural relationships between managers and employees at adjacent levels. An example of the relationship is depicted in Figure \ref{fig.manager}. Specifically, for any level-$\ell$ manager with bin $I_{\ell, m}$, at the next level the bin will be split into two smaller bins $I_{\ell+1,2m-1}$ and $I_{\ell+1,2m}$. When the private value $v_t$ falls into one of the smaller bins, the manager will be follow the advice from one of the four employees in that smaller bin: the \emph{upper} employee, the \emph{middle} employee, the \emph{lower} employee, and the \emph{dummy} employee, represented by the intervals/bids $J_{\ell+1,2u+1}, J_{\ell+1,2u}, J_{\ell+1,2u-1}$ and $b_{\ell+1,u}$, respectively. Note that here the employees have a product structure as the manager can choose from $4\times 4=16$ employees based on her choice of employees in each smaller bin. In each smaller bin, the introduction of upper, middle, and lower employees is motivated by the covering $J_{\ell,u} = J_{\ell+1,2u+1}\cup J_{\ell+1,2u}\cup J_{\ell+1,2u-1}$ and the following lemma. 

\begin{figure}[!t]
\centering
\begin{tikzpicture}
\draw [dashed] (0,0) -- (0, 9);
\draw [dashed] (6,0.8) -- (6, 8.5);
\draw [dashed] (12,0) -- (12, 9);
\draw [thick, blue, ->] (5,1.5) -- (5,1.9); \node [below, blue] at (5,1.5) {$v_t$};
\draw [thick] (0,2) -- (12,2);
\draw [thick] (0,8) -- (12,8); 
\draw [->] (11,6) -- (11,7.8); \draw [->] (11,4) -- (11,2.2);
\node at (11,5.25) {Manager}; \node at (11,4.75) {$(\ell, m_{\ell}^\star, u)$};
\draw [->] (2,1) -- (0.2,1); \draw [->] (4,1) -- (5.8,1); \node at (3,1) {$I_{\ell+1,m_{\ell+1}^\star}$};
\draw [->] (5,0.2) -- (0.2,0.2); \draw [->] (7,0.2) -- (11.8, 0.2); \node at (6,0.2) {$I_{\ell,m_{\ell}^\star}$};

\draw [red, dashed] (0,7.9) -- (6,7.9); \draw [red, dashed] (0,6.5) -- (6,6.5); 
\draw [red, dashed] (0,5) -- (6,5); \draw [red, dashed] (0,3.5) -- (6,3.5); 
\draw [red, dashed] (0,2.1) -- (6,2.1); 
\draw [decorate,red,decoration={brace,amplitude=10pt},xshift=-4pt,yshift=0pt] (0,3.5) -- (0,6.5);
\draw [decorate,red,decoration={brace,mirror,amplitude=10pt},xshift=4pt,yshift=0pt] (6,2.1) -- (6,5);
\draw [decorate,red,decoration={brace,mirror,amplitude=10pt},xshift=4pt,yshift=0pt] (6,5) -- (6,7.9);
\node [red] at (-2,5.25) {Middle Employee}; \node [red] at (-2,4.75) {$(\ell+1, m_{\ell+1}^\star, 2u)$};
\node [red] at (8.2,3.75) {Lower Employee}; \node [red] at (8.2,3.25) {$(\ell+1, m_{\ell+1}^\star, 2u-1)$};
\node [red] at (8.2,6.75) {Upper Employee}; \node [red] at (8.2,6.25) {$(\ell+1, m_{\ell+1}^\star, 2u+1)$};
\draw [red, ->] (-0.5, 7.9) -- (-0.1, 7.9);
\node [red] at (-2,8.15) {Dummy Employee}; \node [red] at (-2, 7.65) {$(\ell+1, m_{\ell+1}^\star, u)$};
\end{tikzpicture}
\caption{A pictorial illustration of managers/employees at different levels. Based on each manager at level $\ell$, the bin $I_{\ell, m_{\ell}^\star}$ is divided into two small bins, where on each small bin the manager has $4$ employees known as the upper employee, the middle employee, the lower employee, and the dummy employee.} \label{fig.manager}
\end{figure}
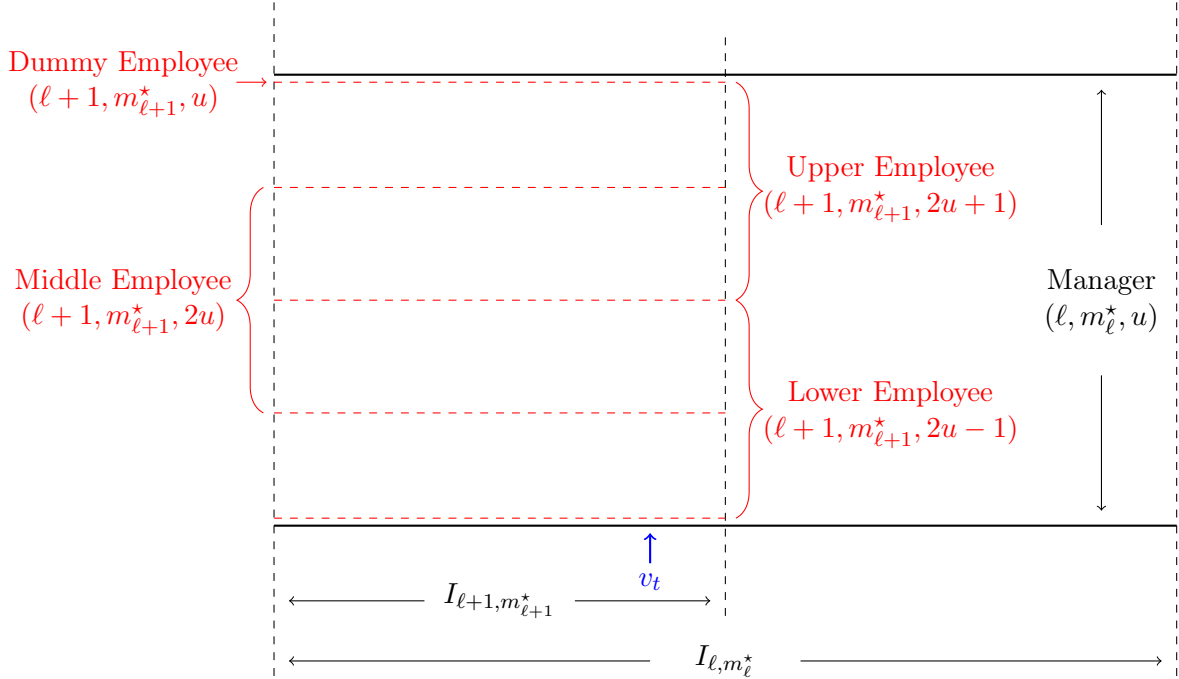

\begin{lemma}\label{lemma:employee}
Let $\ell\in [L-1], m\in [M_{\ell}], u\in [U_{\ell}]$. Then for any $f\in \calF_{\text{\rm Lip}}$ with $f(I_{\ell, m})\subseteq J_{\ell,u}$, and any $m'\in \{2m-1,2m\}$, there exists $u'\in \{2u-1,2u,2u+1\}$ such that $f(I_{\ell+1,m'}) \subseteq J_{\ell+1,u'}$. 
\end{lemma}

Lemma \ref{lemma:employee} implies that for any bidding strategy used by the oracle, its restriction on each small bin is contained in the support of some mixed strategy used by an employee. Hence, using exponential weights, it could be expected that the mixed strategy of the above employee is competitive to that of the oracle restricted to the given bin. Hence, the key reason to introduce the upper, middle, and lower employees is to fulfill the covering condition in Lemma \ref{lemma:employee} which further ensures that at least one of the above employees is competitive to the oracle. The additional dummy employee is mostly for technical purposes (similar to the dummy expert in Section \ref{subsec:chaining}) who serves as a good employee compared with the other three employees with suboptimality gap at most $\Delta_{\ell+1}\le 2^{-\ell}$. 

Now we are ready to explain the SEW algorithm sequentially. We keep track of the cumulative rewards $R_{\ell,m,u}$ of all regular experts (i.e. upper, middle, or lower employees), as well as those $R_{\ell,m,w}'$ for all dummy experts. Meanwhile, to determine the learning rate at each time, we also keep track of the visiting time $T_{\ell,m}$ for each bin $I_{\ell,m}$ (note that a time-varying learning rate is necessary here, as the bidder does not know in advance how many times a given bin will be visited in total). Now each round is decomposed into three steps: 
\begin{enumerate}
	\item Compute the exponential weights: First, the bidder receives the private value $v_t$ at the beginning of time $t$ and computes the EW probabilities at each level. Specifically, for each level $\ell\in [L]$, the bidder identifies the unique bin $I_{\ell,m_\ell^\star}$ which $v_t$ belongs to. In this bin, there are $W_\ell=2^{\ell}-1$ groups of experts, where each group consists of $4$ employees and corresponds to a single manager at the previous level (cf. Figure \ref{fig.manager}). Within each group $w\in [W_{\ell}]$, the bidder computes the probability vector $p_{\ell,m_{\ell}^\star,w} = \{p_{\ell,m_{\ell}^\star,w}(s)\}_{s=1,2,3,4}$ of choosing each employee based on the EW subroutine in Algorithm \ref{algo:ew}, where the learning rate is given by the current visiting time $T_{\ell,m}$ of the current bin (cf. \eqref{eq:EW_prob}). We remark that only one bin is considered at each level, and nothing needs to be done in other bins. 
	\item Random action based on the exponential weights: Based on the above EW probabilities, the bidder is now in a position to make a randomized bid in a top-down order. Specifically, the bidder starts at level $\ell=1$, finds the bin $I_{1,m_1^\star}$ which $v_t$ belongs to, and tosses a random coin to decide which employee to follow in the only group of experts (as $W_1=1$) according to the EW probability $p_{1,m_1^\star,1}(\cdot)$. If the dummy employee is selected, then the bidder bids the price $b_{1,1}$ corresponding to the dummy employee. Otherwise, the bidder continues to level $2$ and locates the group of experts corresponding to the chosen employee at level $1$. Again, the bidder follows a random employee in the new group based on the EW probability within this group, and continues this process. This process terminates when a dummy expert is selected, or the last level $\ell=L$ has been reached where each employee $(L,m,u)$ deterministically bids the midpoint $2^{-L-1}u$ of the interval $J_{L,u}$. 
	\item Update the rewards: Finally, the bidder observes the HOB $m_t$, and then updates the rewards of all (regular or dummy) experts in a bottom-up order. Specifically, the bidder starts from level $\ell= L$ and obtains the instantaneous rewards for all level-$L$ experts who make deterministic bids. Now the bidder moves to level $L-1$, where all regular experts at this level randomly follow the experts at level $L$ with a known probability distribution. Hence, the instantaneous rewards of all level-$(L-1)$ regular experts can also be computed according to \eqref{eq:update_reward}. This process can be continued until all instantaneous rewards are obtained, and then the cumulative rewards can be computed accordingly. 
\end{enumerate}

In summary, the SEW policy runs the exponential weighting at each level for picking the expert in the next level, and updates the rewards of a small portion of experts, i.e. those involved in the corresponding bin. The performance of the SEW policy is summarized in the following theorem, which completes the proof of Theorem \ref{thm:computation}. 
\begin{theorem}\label{thm:SEW}
	The {\rm SEW} policy takes $O(T)$ space and $O(T^{3/2})$ time, and 
	\begin{align*}
	\bE[R_T(\pi^{\text{\rm SEW}})] \le (2+4C(1+2\log_2 T))\sqrt{T}, 
	\end{align*}
	where $C>0$ is the absolute constant appearing in Theorem \ref{thm:good_expert}. 
\end{theorem}

\section{Further Discussions}\label{sec:discussion}
\subsection{Non-Lipschitz Reward}\label{subsec:non-lipschitz}
The only property of the reward function $r(b;v,m) = (v-b)\1(b\ge m)$ used by our policy is that the mapping $b\mapsto r(b;v,m)$ satisfies the one-sided Lipschitz property: $r(b;v,m) - r(b';v,m) \ge - (b-b')$ for all $v\ge b\ge b'$ (cf. Lemma \ref{lemma:continuity}). In other words, the mapping $b\mapsto r(b;v,m)$ from the feasible action $b\in [0,v]$ to its associated reward has right derivative at least $-1$. In particular, we remark that the above mapping is \emph{not} Lipschitz in general, for this function has a jump and is discontinuous at $b=m$. This remarkable difference differentiates our work from \cite{cesa2017algorithmic}, renders their techniques inapplicable, and highlights the necessity of the good expert. In fact, if the reward function \emph{were} Lipschitz on the action, then the chaining arguments in \cite{cesa2017algorithmic} lead to the regret
\begin{align*}
O\left( \sqrt{T} + \sum_{m=1}^M \Delta_{m-1}\sqrt{\frac{T}{\varepsilon_m}}\right),
\end{align*}
for all experts with the same manager at level $m$ receive rewards within a range of $\Delta_m$. In contrast, with only the one-sided Lipschitz property, all experts with the same manager at level $m$ may still receive rewards within a range of $\Theta(1)$, but the presence of a good expert leads to a larger regret
\begin{align*}
O\left( \sqrt{T} + \sum_{m=1}^M \sqrt{\Delta_{m-1}\cdot \frac{T}{\varepsilon_m}}\right),
\end{align*}
with the dependence on the suboptimality gap inside the squared root. Hence, the ChEW (and also SEW) policy still achieves an $\widetilde{O}(\sqrt{T})$ regret, which holds for general problems with the full information and a one-sided Lipschitz reward. 

	\subsection{Generalization to Monotone Oracles}\label{subsec:monotone_oracle}
In addition to the Lipschitz oracle $\calF_{\text{Lip}}$ considered throughout the paper, there is also another natural choice of the oracle, i.e. the monotone oracles $\calF_{\text{Mono}}$ consisting of all functions $f: [0,1]\to [0,1]$ such that $v\mapsto f(v)$ is monotonically increasing. This class of the monotone oracle is very practical as it includes common bidding strategies which start to join the auction only if the private value exceeds some threshold. Also note that in the stochastic case where others' highest bids $m_t$ follow any \emph{iid} distribution, it was shown in \cite[Lemma 5]{han2020optimal} that the oracle who knows the distribution of $m_t$ always uses a monotone bidding strategy, again validating this assumption. 

However, the generalization to monotone oracles is not straightforward. The main reason is that the function class $\calF_{\text{Mono}}$ is not \emph{totally bounded} in the $L^\infty([0,1])$-norm, as the functions $\{ \1(v\ge \alpha) \}_{\alpha\in [0,1]}$ are $L^\infty$-separated from each other. Surprisingly, we have the following impossibility result showing that the worst-case regret against the monotone oracle is $\Omega(T)$. 

\begin{theorem}\label{thm:monotone}
	There exists an oblivious adversarial sequence $(v_t,m_t)_{t\in [T]}$ such that any policy $\pi$ has an expected regret at least $\Omega(T)$ against the monotone oracle. 
\end{theorem}

Theorem \ref{thm:monotone} shows that the non-compactness of $\calF_{\text{Mono}}$ under $L^\infty([0,1])$ leads to a linear regret. Nevertheless, under any $L_p$ norm with $p\in [1,\infty)$, $\calF_{\text{Mono}}$ becomes totally bounded again. 

\begin{lemma}[\!\!\cite{birman1967piecewise}]\label{lemma:covering_monotone}
	For each $p\in [1,\infty)$, there exists a constant $C_p>0$ such that the $\varepsilon$-bracketing number of $\calF_{\text{\rm Mono}}$ is upper bounded by $\exp(C_p/\varepsilon)$ for all $\varepsilon\in (0,1)$. 
\end{lemma}

Based on Lemma \ref{lemma:covering_monotone}, we prove that when the private values are stochastic with mild assumptions on the density, then an $\widetilde{O}(\sqrt{T})$ average regret can again be achieved against any monotone oracles. 

\begin{theorem}\label{thm:monotonce_oracle}
	Fix any $q\in (1,\infty]$ and $L>0$. Let the private values $(v_1,\cdots,v_T)$ be drawn from some unknown distribution $P$ with marginals $P_1,\cdots,P_T$, and $P_t$ admits a density $p_t$ on $[0,1]$ with $\|p_t\|_q \le L$ for all $t\in [T]$. Then there exists a policy $\pi=(b_1,\cdots,b_T)$ such that
	\begin{align*}
		\max_{f\in\calF_{\text{\rm Mono}}} \bE\left[ \sum_{t=1}^T \left( r(f(v_t);v_t,m_t) - r(b_t;v_t,m_t) \right) \right]\le C\sqrt{T}\log T, 
	\end{align*}
	where $(m_1,\cdots,m_T)$ is any adversarial sequence in $[0,1]$, the expectation is taken jointly over the randomness of the private values and the bidder's policy, and constant $C>0$ depends only on $(q,L)$. 
\end{theorem}

\section{Experiments on Real Data}\label{sec:experiment}
We now demonstrate the performance of our proposed bidding policy on three real auction datasets obtained from Verizon Media. To streamline the presentation, we first present in Section~\ref{subsec:exp_data} an overview and a visualization of the datasets, and then introduce in Section~\ref{subsec:exp_policy} three other competing bidding policies for comparison. Finally, for the experimental results, we show in Section~\ref{subsec:exp_result} that each of the competing policy, despite performing well in certain datasets, behaves poorly in others, while in contrast, our SEW policy enjoys superior performance on all datasets and uniformly outperforms both competing policies.

\subsection{Data Description}\label{subsec:exp_data}
Our experiments are run on a total number of three auction datasets from the first-price auctions on three real-world sites, where each dataset consists of the bidding data through the Verizon Media demand side platform (DSP) during a one-month period from March 24, 2020 to April 22, 2020. For business confidentiality, we do not intend to disclose the full data, nor the identities of the real-world sites; we will refer to datasets A, B, and C instead. These datasets consist of around 0.54, 1.00, and 1.57 million data points, respectively, where each data point is a pair of scalars $(v_t, m_t)$ including the private value $v_t$ and the HOB $m_t$ for each auction. The private value $v_t$ is computed by Verizon Media based on an independent learning scheme not relying on the auction, and is therefore taken as given. The HOB $m_t$ is returned by the platform after each auction, possibly including the seller's reserve price and measured up to 1 cent. These datasets have already been pruned to only contain data points with $v_t > m_t$, for otherwise the bidder never wins regardless of her bids. 

\begin{figure}[!t]
	\centering
	\begin{subfigure}[b]{0.32\textwidth}
		\includegraphics[width=\linewidth]{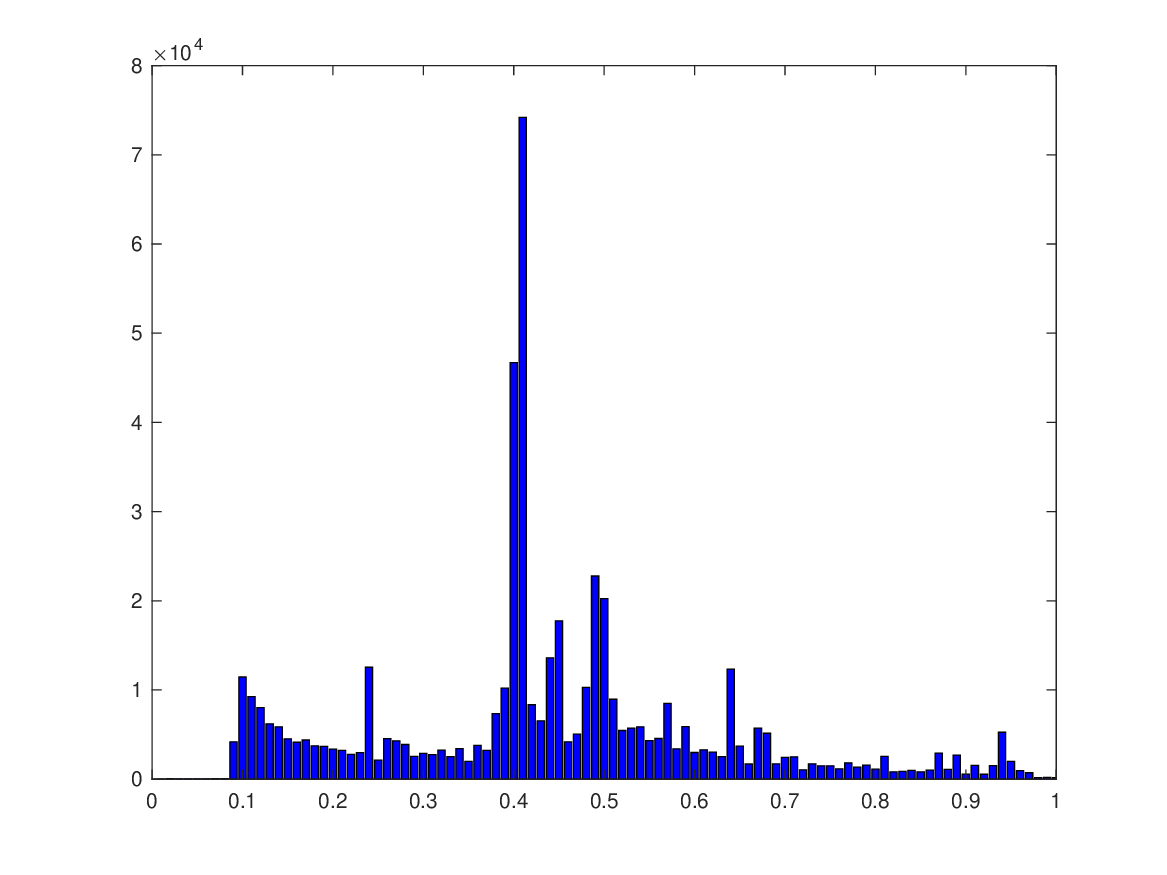}\caption{Dataset A.}
	\end{subfigure}
	\begin{subfigure}[b]{0.32\textwidth}
		\includegraphics[width=\linewidth]{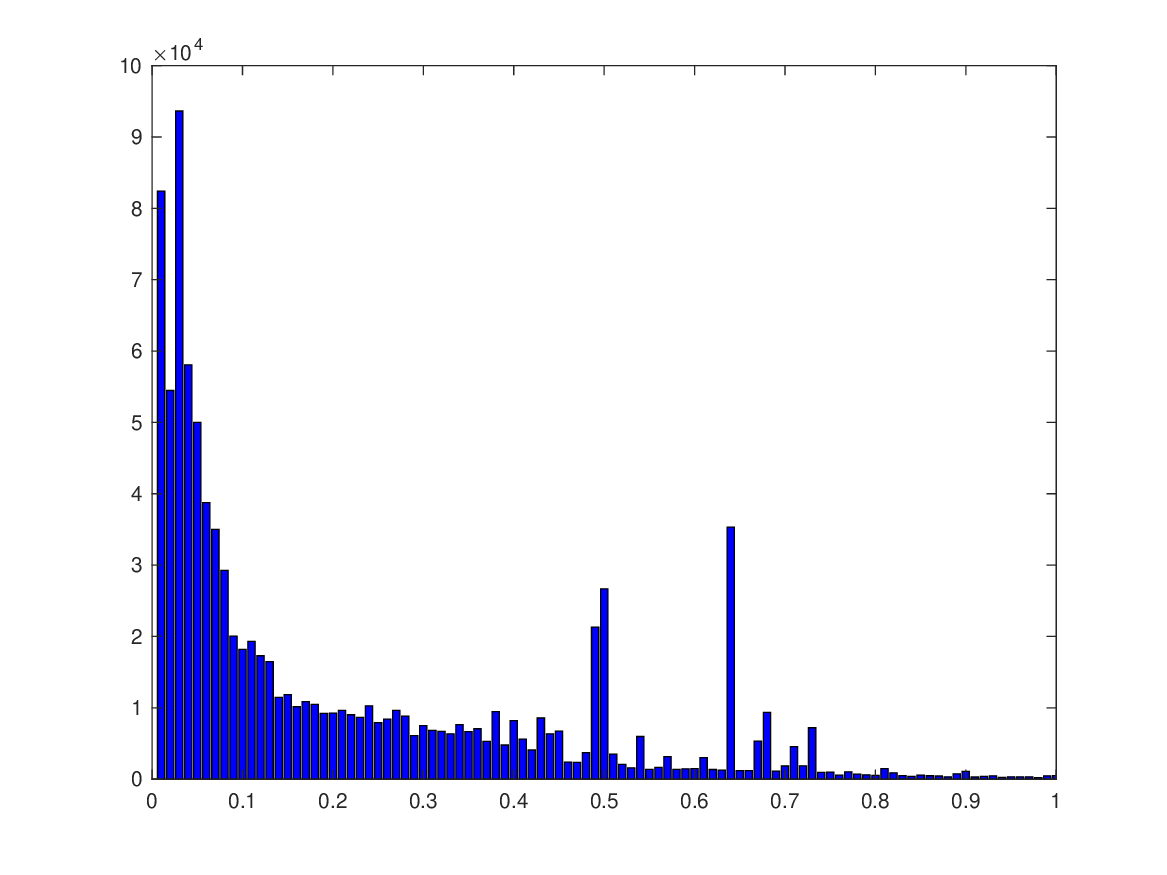}\caption{Dataset B.}
	\end{subfigure}
	\begin{subfigure}[b]{0.32\textwidth}
		\includegraphics[width=\linewidth]{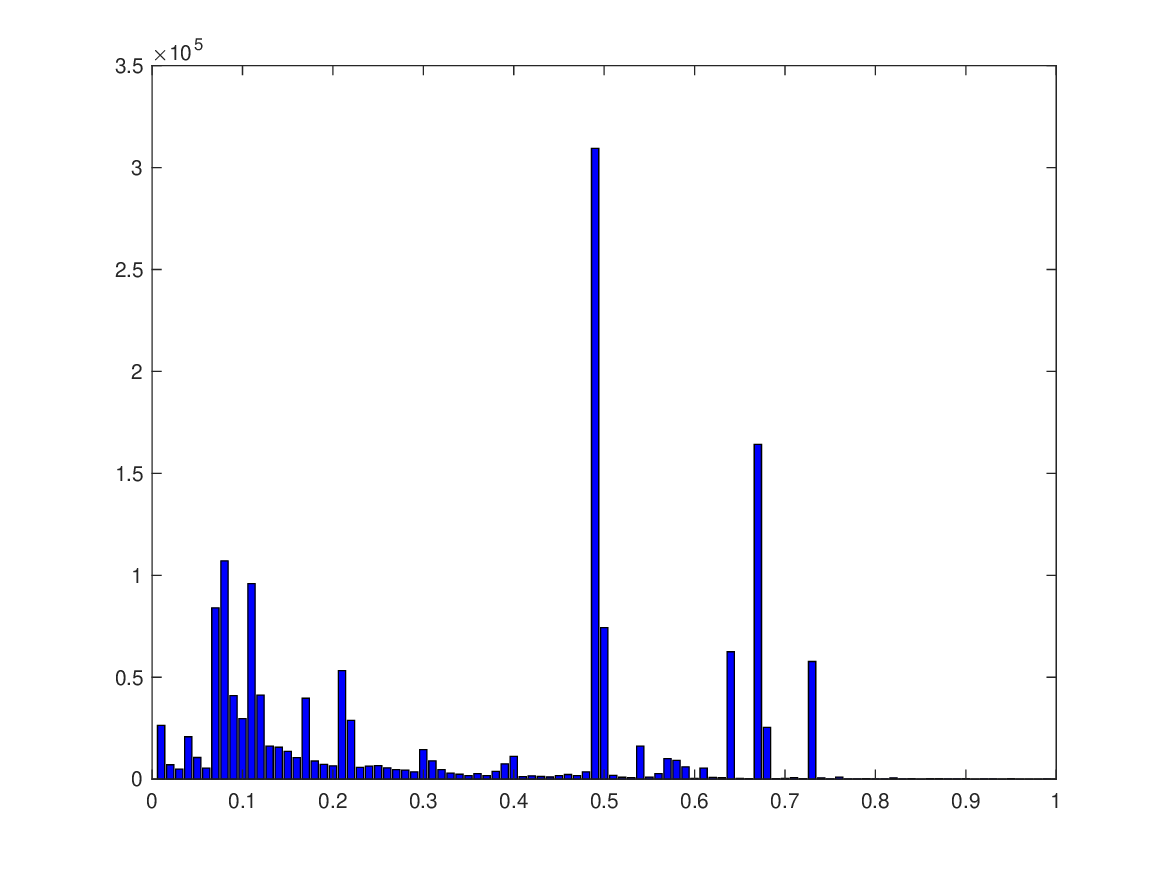}\caption{Dataset C.}
	\end{subfigure}
	\caption{The histograms of the private value sequence $(v_t)$ in the three datasets, with the normalization into $v_t\in [0,1]$ and $100$ equal-spaced bins.}\label{fig:data_v}
\end{figure}

\begin{figure}[!t]
	\centering
	\begin{subfigure}[b]{0.32\textwidth}
		\includegraphics[width=\linewidth]{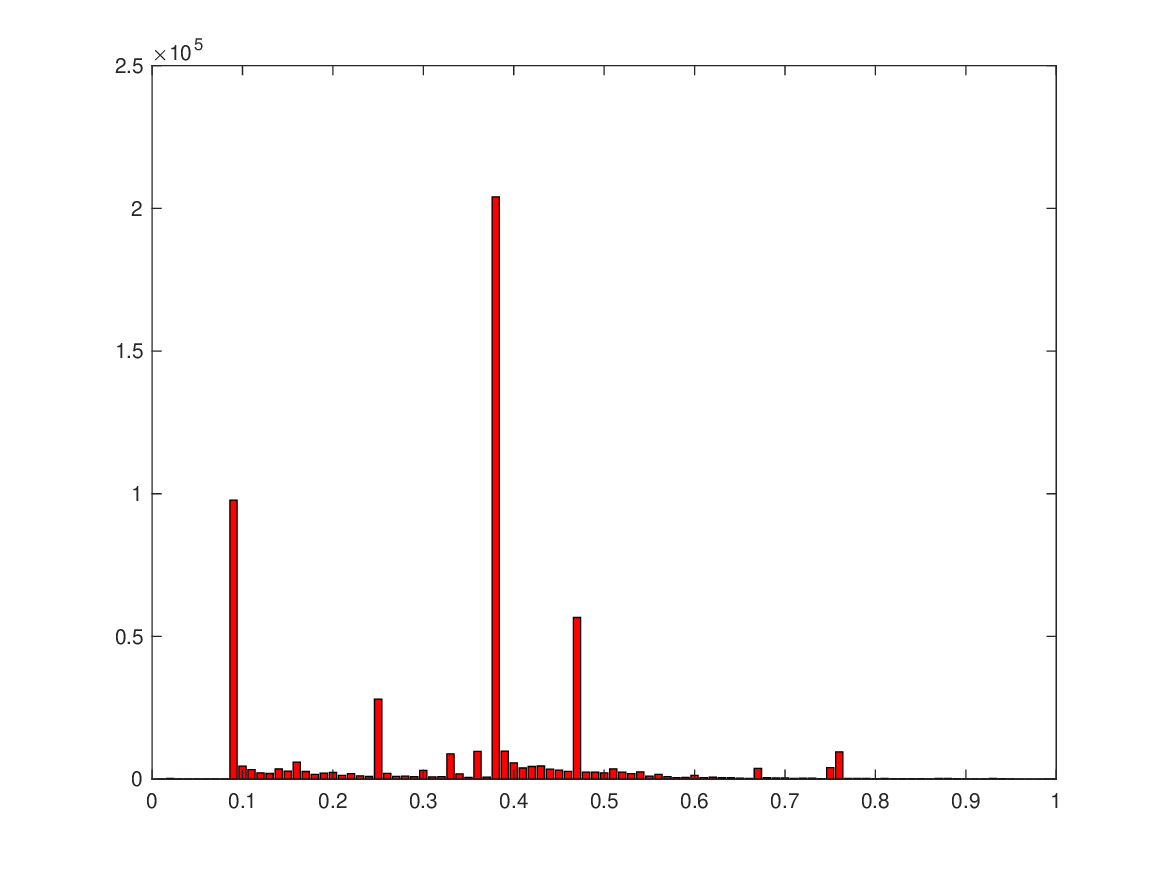}\caption{Dataset A.}
	\end{subfigure}
	\begin{subfigure}[b]{0.32\textwidth}
		\includegraphics[width=\linewidth]{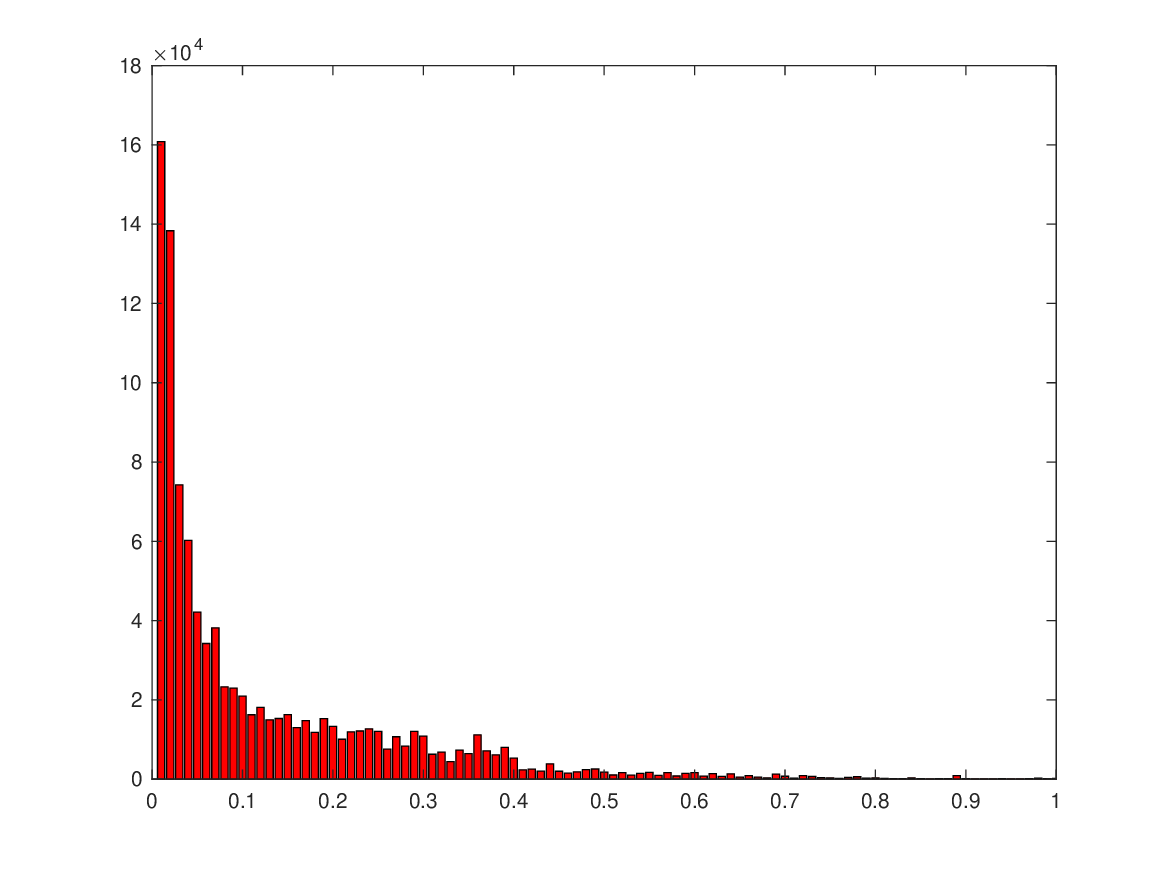}\caption{Dataset B.}
	\end{subfigure}
	\begin{subfigure}[b]{0.32\textwidth}
		\includegraphics[width=\linewidth]{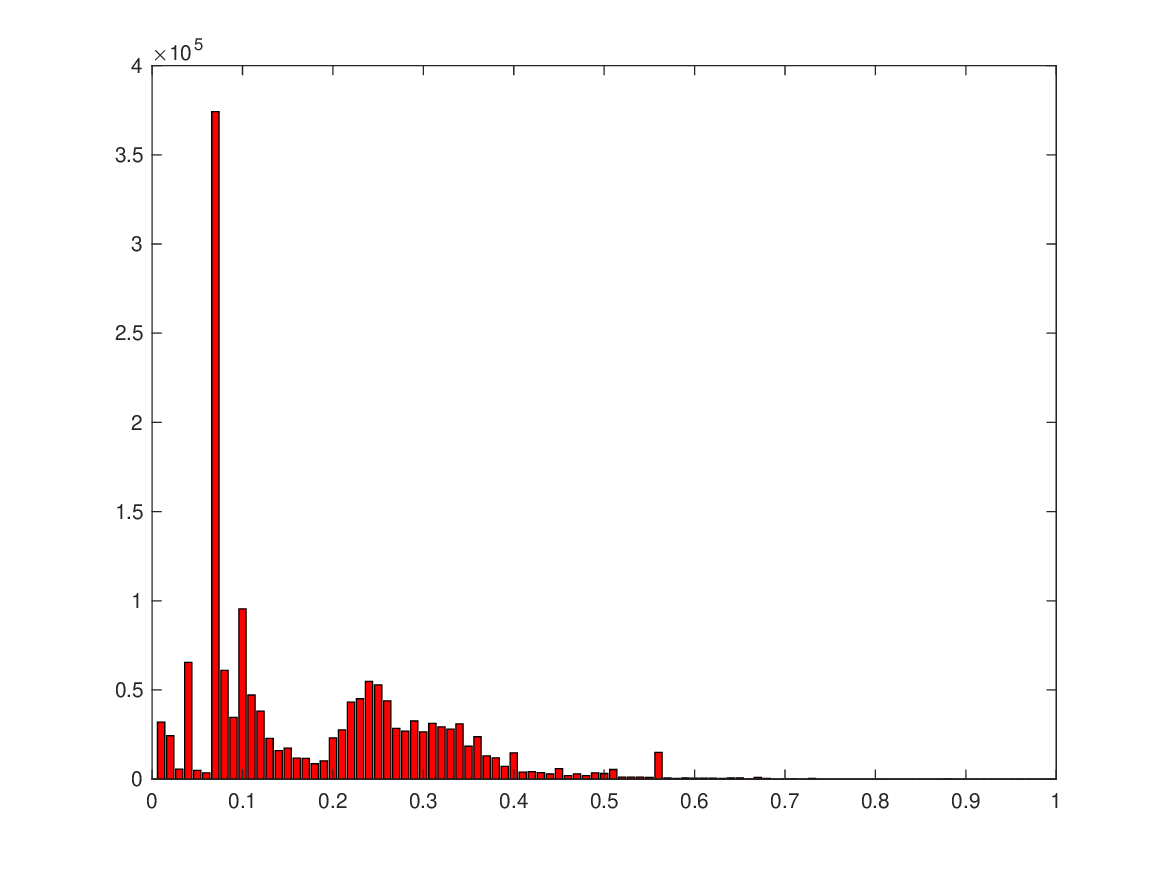}\caption{Dataset C.}
	\end{subfigure}
	\caption{The histograms of the minimum winning bid sequence $(m_t)$ in the three datasets, with the normalization into $m_t\in [0,1]$ and $100$ equal-spaced bins.}\label{fig:data_m}
\end{figure}

Although we reveal little information about the datasets for business confidentiality, we provide a visualization to illustrate the important points. Specifically, Figures \ref{fig:data_v} and \ref{fig:data_m} display the histograms of the private values $v_t$ and the minimum bids $m_t$ needed to win in three datasets, respectively, where only the relative values after normalization are presented. We observe that histograms in different datasets exhibit different properties. For example, Figure \ref{fig:data_v} shows that the empirical distribution of the private values in dataset C has a large discrete component (i.e. taking relatively few values), whereas those in the datasets A and B are closer to continuous distributions. Similarly, Figure \ref{fig:data_m} shows that HOBs $m_t$ have a near-discrete distribution in dataset A, while in the other two datasets the distributions of $m_t$ are more continuous. As we shall see in the next subsection, these differences lead to varying performances for the two competing bidding policies.

\subsection{Competing Bidding Policies}\label{subsec:exp_policy}
In the experiments we apply our SEW policy in Algorithm \ref{algo:sew} to all three datasets, with learning rate $\eta_t=5/\sqrt{t\Delta}$ in Algorithm \ref{algo:ew} and proper scalings to accommodate real ranges of the private values and the candidate bids instead of $[0,1]$. Below we introduce three competing bidding policies covering both ideas of parametric modeling and nonparametric learning. 

\begin{enumerate}
\item\textbf{Competing Policy $1$: Linear Bid Shading Policy.}
The idea of \emph{bid shading}, referring to the fact that bidders should bid less than their private valuation, is well-known in first-price auctions \cite{bid_shading}. In general, the bid shading idea assumes a parametric model of the optimal bid, where the simplest policy is the linear bid shading \cite{abeille2018explicit}. Specifically, the bidder's bid $b = \theta v$ is assumed to be a linear function of the private value $v$, indexed by some scalar parameter $\theta\in [0,1]$ to be learned. To determine the optimal parameter $\theta$, we collect the past data in a given time window (usually one day) and use a brute force grid search to find the optimal parameter $\theta$ to maximize the total profit in the previous window.

\item \textbf{Competing Policy $2$: Non-linear Bid Shading Policy.} 
The practical performance of the bid-shading policy could be improved by introducing certain non-linearity into the parametric modeling. To this end, we also compare with a recent non-linear bid shading policy proposed in \cite{karlsson2020adaptive} motivated by domain knowledge: the parameter set is $\theta = (\theta_1,\theta_2)\in [0,1] \times [0,\infty)$, and the bid $b$ takes the parametric form
\begin{align*}
b(v;\theta) = \frac{\log(1+\theta_1\theta_2v)}{\theta_2}.
\end{align*}
Since $\log(1+x)\le x$ for all $x>-1$, the above policy achieves the bid shading over the entire parameter set, with parameters $\theta_1, \theta_2$ controlling the degrees of bid shading. As before, we use a brute force grid search to find the optimal low-dimensional parameter $\theta$ to maximize the total profit in the past one-day window and use it for the current window. 

\item \textbf{Competing Policy $3$: Distribution Learning Policy.}
This policy is nonparametric and estimates the probability distribution of HOBs (assuming that HOBs are stationary in a short time window). This idea is motivated by its theoretical optimality in regret established in~\cite{han2020optimal} on stochastic first-price auctions with censored feedback. Specifically, at each time $t$, let $\widehat{P}$ be the empirical distribution of HOBs in a given time window, then current bid $b_t$ is chosen to maximize the expected revenue \emph{if} $m_t\sim \widehat{P}$: 
\begin{align*}
b_t = \arg\max_{b\ge 0}\bE_{m_t\sim \widehat{P}}[(v_t - b)\1(b\ge m_t)].
\end{align*}
Note that this approach is entirely nonparametric, and is expected to have a sound performance if the data is indeed stationary. In the experiments we always choose the length of the time window to be one day, which is roughly the optimal length of the time window in hindsight for all datasets. 
\end{enumerate}

Finally, we remark that as this paper is devoted exclusively to the understanding of the possibly simplest first-price auction model without any external or side information, we are not comparing with bidding policies aided by various sources of side information available in practice. The effects of different kinds of side information on the auction performance were later studied in \cite{zhang2022leveraging}.

\subsection{Experimental Results}\label{subsec:exp_result}
We plot and compare the cumulative rewards (normalized to $[0,1]$ to avoid information leakage) of all four policies as a function of time in all datasets in Figure \ref{fig:result}. We present a quick summary of our key findings next:

\begin{figure}[t]
	\centering
	\begin{subfigure}[b]{0.32\textwidth}
		\includegraphics[width=\linewidth]{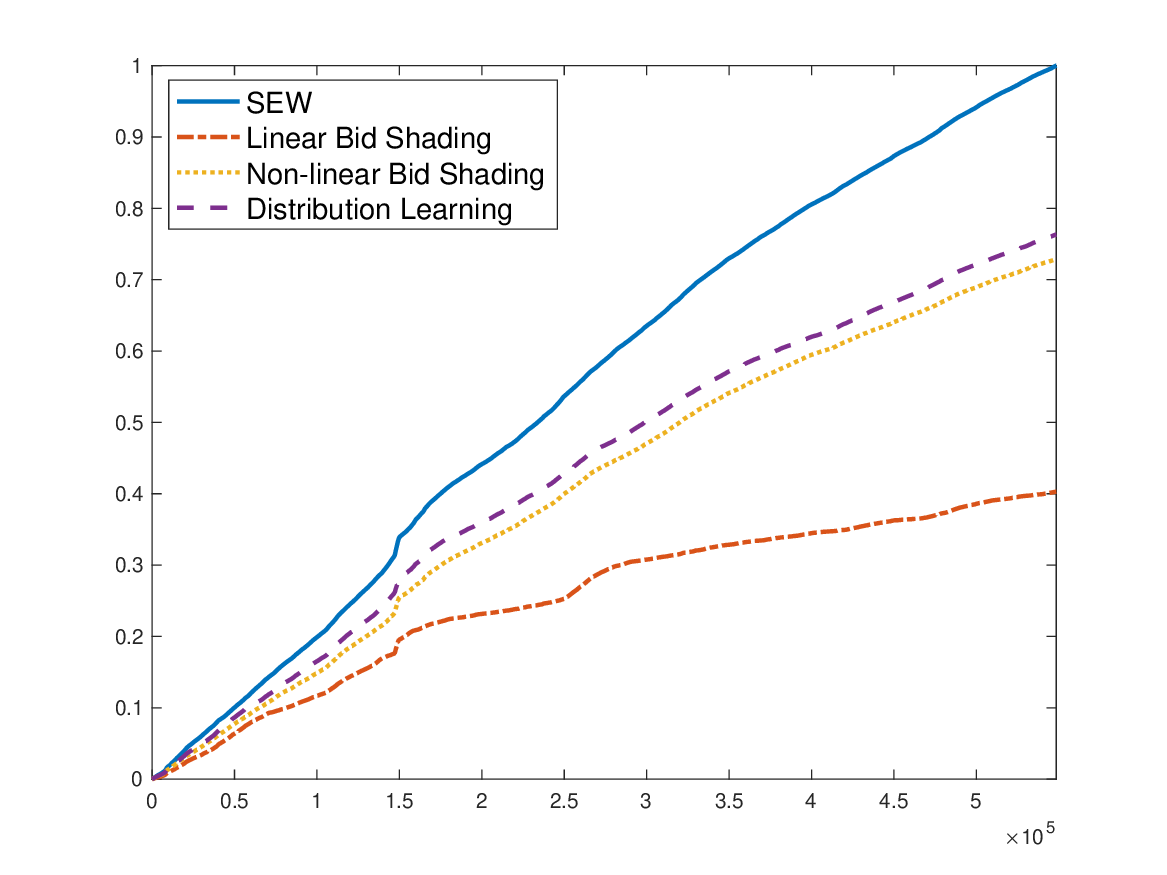}\caption{Dataset A.}
	\end{subfigure}
	\begin{subfigure}[b]{0.32\textwidth}
		\includegraphics[width=\linewidth]{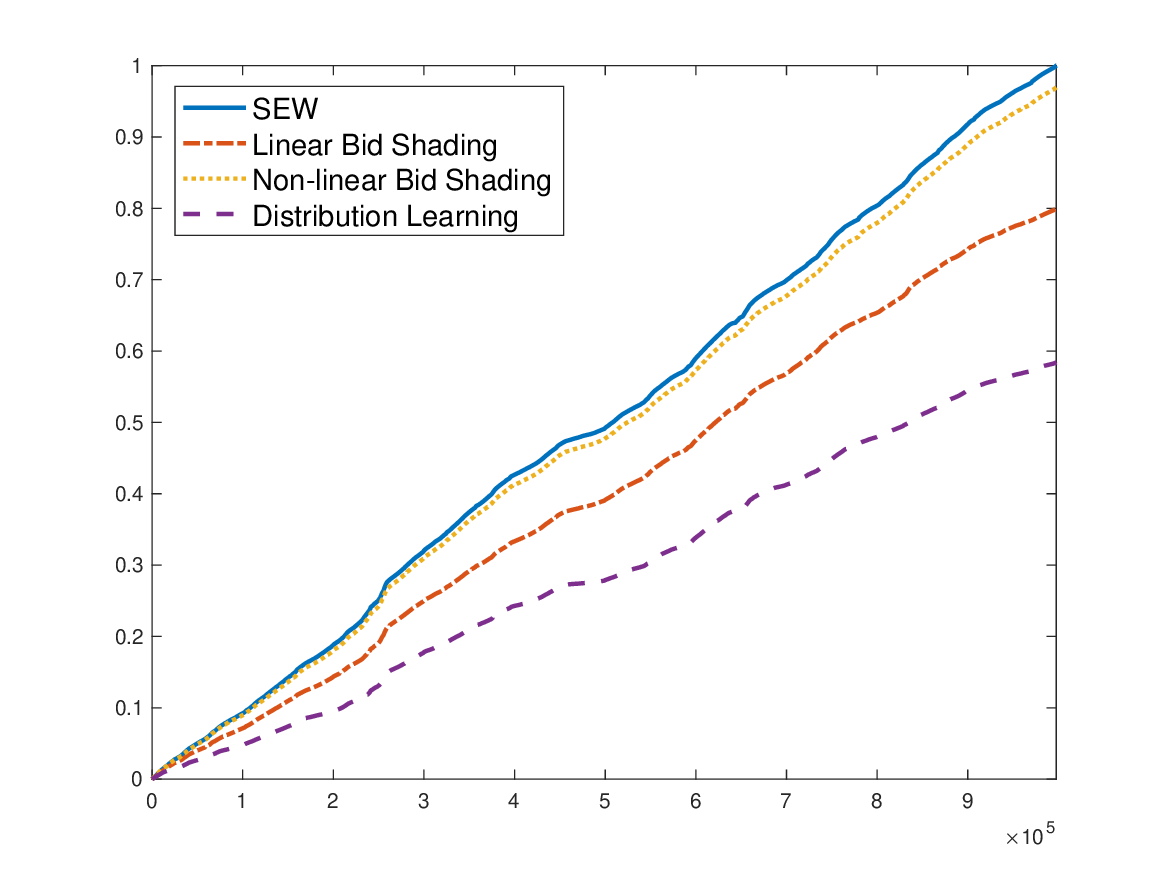}\caption{Dataset B.}
	\end{subfigure}
	\begin{subfigure}[b]{0.32\textwidth}
		\includegraphics[width=\linewidth]{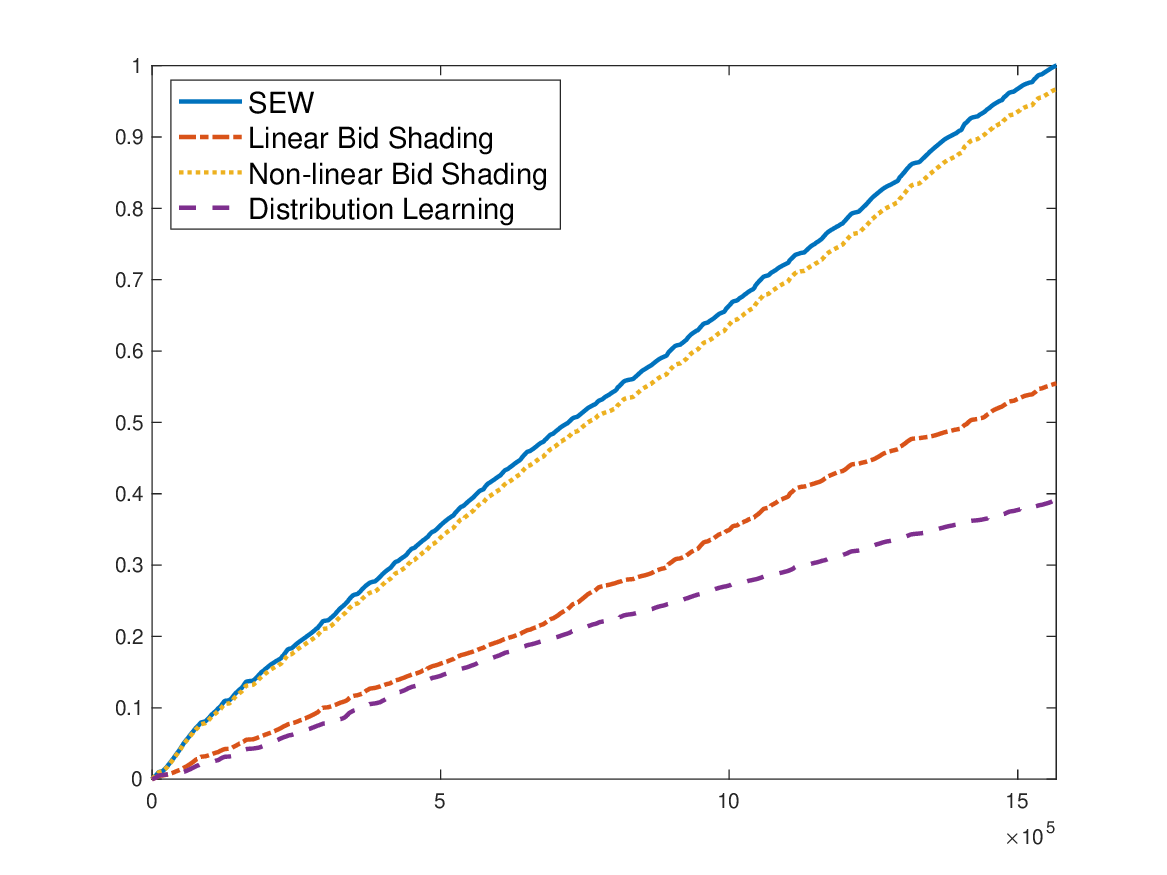}\caption{Dataset C.}
	\end{subfigure}
	\caption{The cumulative rewards (normalized to $[0,1]$) as a function of time in all datasets, where solid lines correspond to the SEW policy, dashdot and dotted lines correspond to the linear and non-linear bid shading policies, respectively, and dashed lines correspond to the distribution learning policy.}\label{fig:result}
\end{figure}

\begin{figure}[t]
	\centering
	\begin{subfigure}[b]{0.32\textwidth}
		\includegraphics[width=\linewidth]{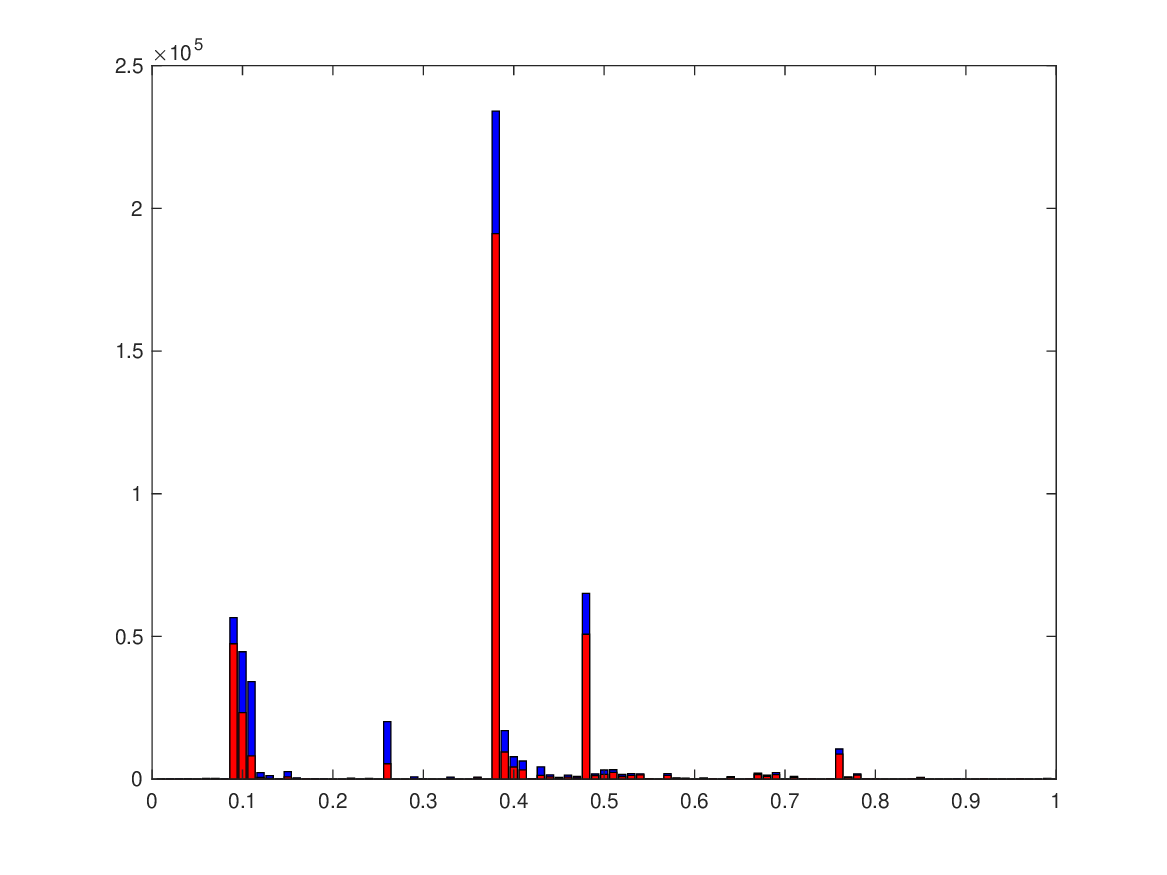}\caption{SEW.}
	\end{subfigure}
	\begin{subfigure}[b]{0.32\textwidth}
		\includegraphics[width=\linewidth]{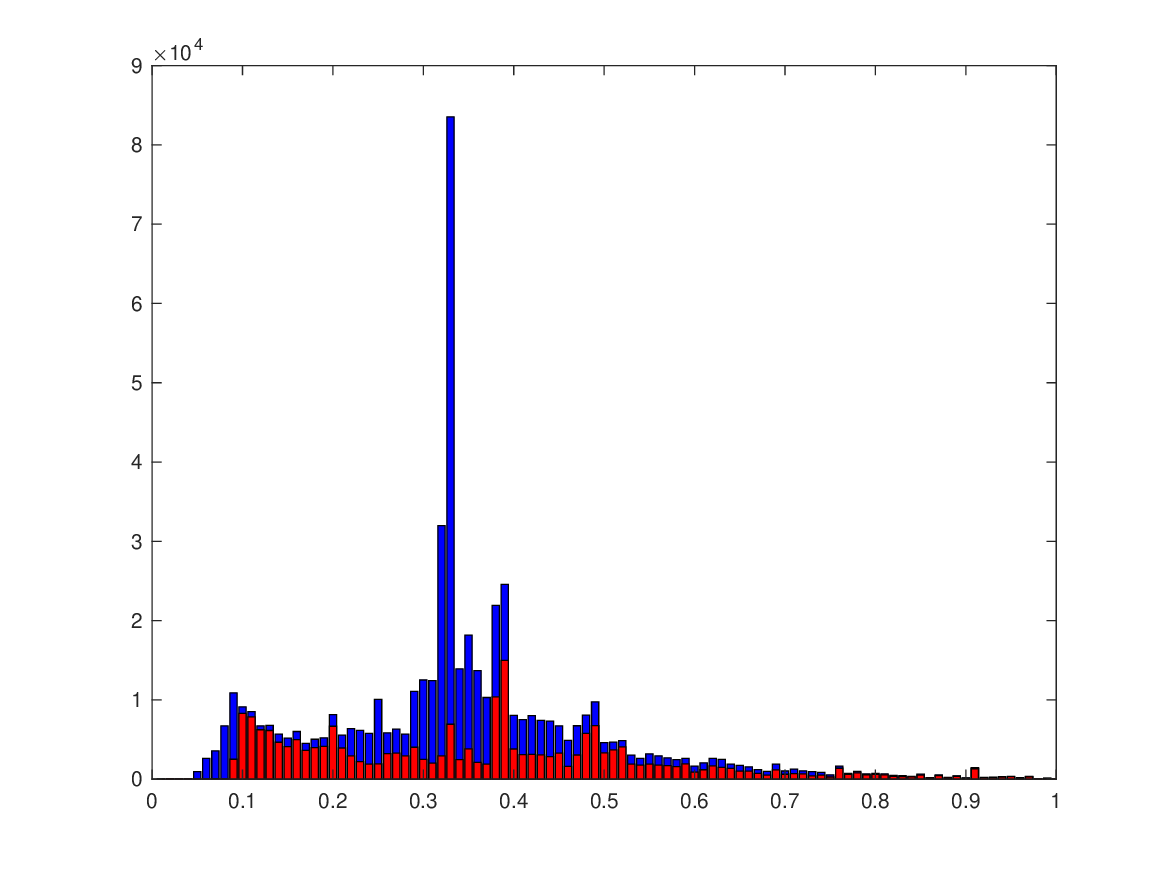}\caption{Non-linear Bid Shading.}
	\end{subfigure}
	\begin{subfigure}[b]{0.32\textwidth}
		\includegraphics[width=\linewidth]{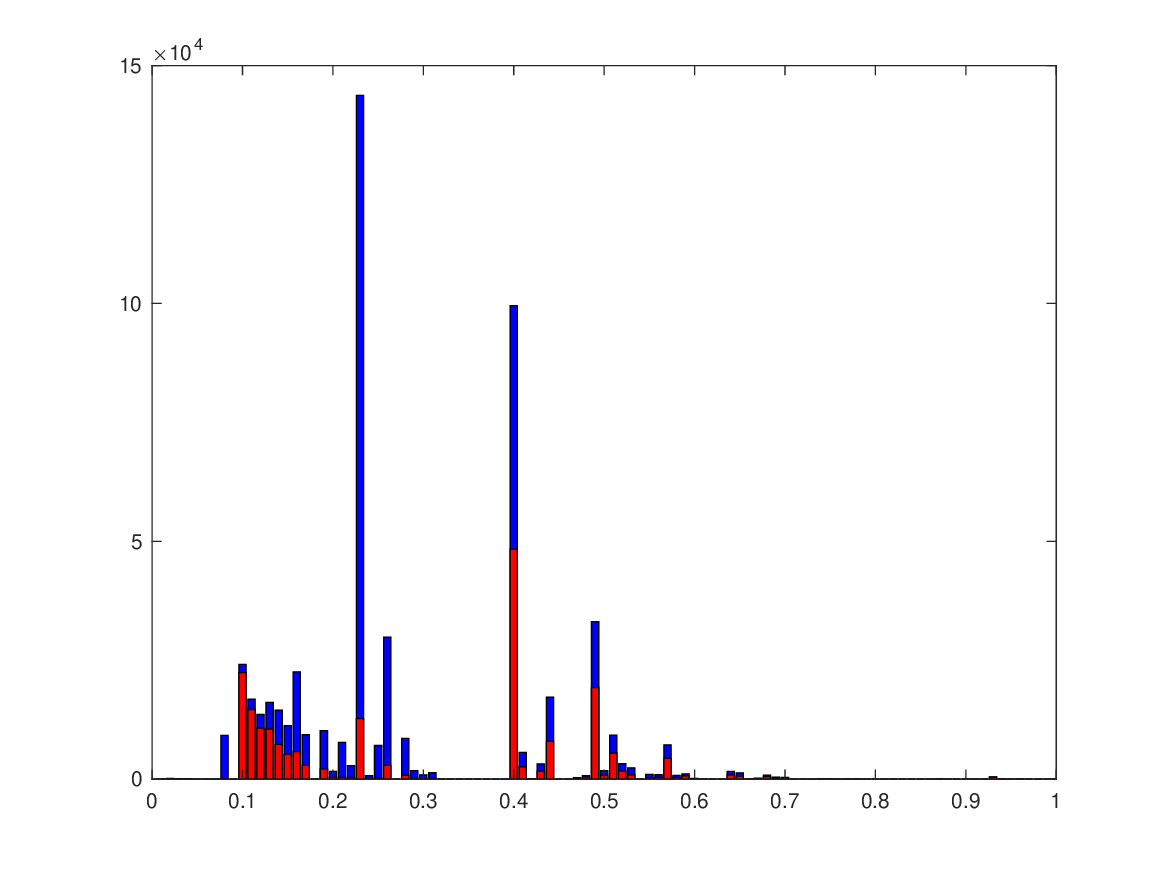}\caption{Distribution Learning.}
	\end{subfigure}
	\caption{The bar plots of the total bids and winning bids for the SEW, non-linear bid shading, and distribution learning policies in dataset A, where the height of the blue rectangle represents the number of bids made in the given range, and the height of the red rectangle represents the number of wins achieved by the bid.}\label{fig:bar_plot}
\end{figure}

\begin{itemize}
	\item The non-linear bid shading policy outperforms the linear one in all datasets and achieves good rewards in both datasets B and C, but it performs poorly in dataset A. To see why, recall that Figure \ref{fig:data_m} shows a distinguishing feature of dataset A, i.e. most HOBs are supported on a few prices. In this case, a good policy should identify these prices and bid a bit higher than one of them, and the HOB is the key contributing factor to the optimal bid. In contrast, any bid shading policy results in bids relying too much on the private values and fails to capture the discrete structure of HOBs. This phenomenon is further illustrated in Figure \ref{fig:bar_plot}, where we plot the histograms of the total and winning bids for each policy in dataset A. We observe that the discrete components of the bids made by the SEW policy almost coincide with those in Figure \ref{fig:data_m}(a), and the distribution learning policy also makes mostly discrete bids. In contrast, the bids made by the non-linear bid shading policy are still continuous, resulting in a smaller number of winnings and further a poor total reward\footnote{One may wonder that the \emph{average} reward of each winning bid should be the right target as opposed to the \emph{total} reward, as bidders typically have budget constraints. However, in the current datasets, the budget constraints have already been incorporated in the private values $v_t$ (meaning that $v_t$ is smaller than the true valuation in view of the budget constraint), and this process is independent of the bidding process. Therefore, the total reward is a more appropriate criterion here.}. 
	\item The distribution learning policy has a good performance in dataset A but performs worse in the other two datasets. There are two reasons for this observation. First, when the distribution of $m_t$ becomes more continuous, it is harder to estimate. Hence, for dataset A where the support of $m_t$ is small, a better distribution learning performance is available, also giving a better reward compared with the bid shading policy. However, the data in the other two datasets are more continuous which lead to a poor performance. Second, and more importantly, the real-world data are highly non-stationary. Specifically, the distribution of $m_t$ highly depends on $v_t$, so the estimator of the unconditional distribution of $m_t$ may not be accurate given a specific $v_t$. For example, in dataset A the correlation coefficient between the sequences $v_t$ and $m_t$ is as large as $0.66$, which is far from independence. 

	\item As opposed to the above competing policies which may not work well in certain dataset, the SEW policy is robust to the different natures of the datasets, and uniformly outperforms other policies in all datasets. Specifically, when HOBs have a large discrete component (i.e. in dataset A), the SEW policy can learn this component quickly and achieves a much higher (around $30\%$ larger) total reward than others. Moreover, when HOBs are mostly continuous, the SEW policy can still adapt to the new nature of data and outperforms the non-linear bid shading policy by a margin around $5\%$. 
\end{itemize}

In summary, both the bid shading and distribution learning policies suffer from certain problems in some datasets, while our SEW policy enjoys a robust performance on different types of data and performs uniformly better.

\appendix
\section{Auxiliary Lemmas}
\begin{lemma}[Bernstein's inequality \cite{bennett1962probability}]\label{lemma.bernstein}
	Let $X$ be a random variable with $\bE[X]=0, \var(X) = \sigma^2$ and $|X|\le 1$ almost surely. Then for each $\lambda\ge 0$, 
	\begin{align*}
	\bE[\exp(\lambda X)] \le \exp\left((e^{\lambda} - \lambda - 1)\sigma^2 \right). 
	\end{align*}
	In particular, if $X_1,\cdots,X_n$ are iid random variables with the same distribution as $X$, then
	\begin{align*}
	\bP\left(\left| \frac{1}{n}\sum_{i=1}^n X_i \right| \ge t\right) \le 2\exp\left(-\frac{nt^2}{2(\sigma^2 + t/3)}\right). 
	\end{align*}
\end{lemma}

\begin{lemma}[Fano's inequality \cite{fano1952class}]\label{lemma.fano}
Let $P_1,\cdots,P_n$ be two probability distributions on $(\Omega,\calF)$, and $\Psi: \Omega\to [n]$ be any test. Then
\begin{align*}
\frac{1}{n}\sum_{i=1}^n P_i(\Psi \neq i) \ge 1 - \frac{I(V;X) + \log 2}{\log n},
\end{align*}
where $V\sim \mathsf{Unif}([n])$, and $P_{X|V = i} = P_i$ for all $i\in [n]$. 
\end{lemma}

\section{Deferred Proofs}

\subsection{Proof of Theorem \ref{thm:good_expert}}\label{appendix:good_expert}
\subsubsection{Proof of the Upper Bound}
We show that the upper bound of Theorem \ref{thm:good_expert} holds with $C=132$. Since the exponential-weighting algorithm is symmetric to all experts, we may assume that expert $1$ is good. Moreover, since both the probabilities $p_{t,a}$ and the regret remain unchanged if we replace all instantaneous rewards $r_{t,a}$ by $r_{t,a} - r_{t,1}$, we may assume that $r_{t,1}\equiv 0$ and $r_{t,a}\in [-1,\Delta]$ by the $\Delta$-good assumption. Next, as in the standard analysis of the exponential weighting, we define
\begin{align*}
\Phi_t = \frac{1}{K}\sum_{a=1}^K \exp\left(\eta_{t}\sum_{s<t} r_{s,a}\right), \qquad t=1,\cdots,T+1. 
\end{align*}
To handle the time-varying learning rate, we also define
\begin{align*}
\Phi_{t+1}' =  \frac{1}{K}\sum_{a=1}^K \exp\left(\eta_{t}\sum_{s<t+1} r_{s,a}\right), \qquad t=1,\cdots,T. 
\end{align*}

Then for $t\in [T]$, 
\begin{align}\label{eq:potential_function}
\frac{\Phi_{t+1}'}{\Phi_t} = \sum_{a=1}^K p_{t,a}\cdot \exp(\eta_{t} r_{t,a}) = \bE[\exp(\eta_{t} X_t)], 
\end{align}
where $X_t$ is a random variable taking value $r_{t,a}$ with probability $p_{t,a}$, for all $a\in [K]$. Since $X_t\in [-1,\Delta]$ almost surely, 
\begin{align*}
\var(X_t) &\le \bE[(\Delta - X_t)^2] \\
&\le (1+\Delta) \cdot\bE[\Delta - X_t] \le 2(\Delta - \bE[X_t]),
\end{align*}
and therefore Lemma \ref{lemma.bernstein} implies that
\begin{align}\label{eq:mgf}
\bE[\exp(\eta_{t} X_t)] &\le \exp\left(\eta_{t}\bE[X_t] + \left(e^{\eta_{t}} - \eta_{t} -  1\right)\var(X_t)\right) \nonumber\\
&\le \exp\left(\eta_{t}\bE[X_t] + 2\eta_{t}^2(\Delta - \bE[X_t]) \right),
\end{align}
where the last inequality follows from $\eta_t \le 1$ and $e^x\le 1+x+x^2$ whenever $x\le 1$. Combining \eqref{eq:potential_function} and \eqref{eq:mgf}, we have
\begin{align*}
\frac{\log \Phi_{t+1}'}{\eta_{t}} - \frac{\log \Phi_t}{\eta_{t}} \le \sum_{a=1}^K p_{t,a}r_{t,a} + 2\eta_t\left(\Delta - \sum_{a=1}^K p_{t,a}r_{t,a} \right). 
\end{align*}
Moreover, as $\eta_{t} \ge \eta_{t+1}$, the non-decreasing property of the map $\eta\in \bR \mapsto (n^{-1}\sum_{i=1}^n x_i^\eta)^{1/\eta}$ for any non-negative reals $x_1,\cdots,x_n$ leads to $\eta_{t+1}^{-1}\log\Phi_{t+1}\le \eta_{t}^{-1}\log\Phi_{t+1}'$. Hence, the previous inequality implies that
\begin{align*}
\frac{\log \Phi_{t+1}}{\eta_{t+1}} - \frac{\log \Phi_t}{\eta_{t}} \le \sum_{a=1}^K p_{t,a}r_{t,a} + 2\eta_t\left(\Delta - \sum_{a=1}^K p_{t,a}r_{t,a} \right), 
\end{align*}
and a telescoping argument leads to 
\begin{align}\label{eq:telescope}
\frac{\log \Phi_{t+1}}{\eta_{t+1}} - \frac{\log \Phi_1}{\eta_1} &\le \sum_{s=1}^t\sum_{a=1}^K p_{s,a}r_{s,a} + 2\sum_{s=1}^t \eta_s\left(\Delta - \sum_{a=1}^K p_{s,a}r_{s,a}\right) \nonumber \\
&\le \sum_{s=1}^t(1-2\eta_s)\sum_{a=1}^K p_{s,a}r_{s,a} + 2\sqrt{t\Delta\log K}
\end{align}
for all $t=1,2,\cdots,T$, and in the last inequality we have used $\eta_t\le \sqrt{(\log K)/(\Delta t)}$ and $\sum_{s=1}^t 1/\sqrt{s}\le 2\sqrt{t}$. The previous steps are partially inspired by \cite[Lemma 1]{gyorfi2007sequential}. 

Note that in the classical proof we only need to plug in $t=T$ in \eqref{eq:telescope}; however, here the coefficient of the expected instantaneous reward is changing over time, and it will turn out that the inequality \eqref{eq:telescope} with all $t\in[T]$ is required. The following steps are partially inspired by \cite{auer2002adaptive}. Note that by definition of $\Phi_{t+1}$, we have
\begin{align*}
\log \Phi_{t+1} \ge \eta_{t+1}\cdot \max_{a\in [K]} \sum_{s\le t} r_{s,a} - \log K,
\end{align*}
and therefore \eqref{eq:telescope} with $\Phi_1 = 1$ gives that for $t=1,2,\cdots,T$, 
\begin{align}\label{eq:regret_bound_initial}
\max_{a\in [K]} \sum_{s\le t} r_{s,a} \le \sum_{s=1}^t(1-2\eta_s)\sum_{a=1}^K p_{s,a}r_{s,a} + 4\sqrt{t\Delta\log K} + 4\log K. 
\end{align}
We deduce the desired regret upper bound from the inequality \eqref{eq:regret_bound_initial}. Since $r_{t,1}\equiv 0$, the LHS of \eqref{eq:regret_bound_initial} is always non-negative. Consequently, for all $t\in [T]$, 
\begin{align*}
S_t \triangleq \sum_{s=1}^t(1-2\eta_s)\sum_{a=1}^K p_{s,a}r_{s,a} \ge -4\sqrt{t\Delta\log K} - 4\log K. 
\end{align*}
Hence, with the convention $S_0 \triangleq 0$, we have
\begin{align*}
\sum_{t=1}^T \eta_t \sum_{a=1}^K p_{t,a}r_{t,a} 
&= \sum_{t=1}^T \frac{\eta_t}{1-2\eta_t}(S_t - S_{t-1}) \\
&= \sum_{t=1}^{T-1} S_t\left(\frac{\eta_t}{1-2\eta_t} - \frac{\eta_{t+1}}{1-2\eta_{t+1}} \right) + S_T\cdot \frac{\eta_T}{1-2\eta_T} \\
&\stepa{\ge} \sum_{t=1}^{T-1} -4(\sqrt{t\Delta\log K}+\log K)\left(\frac{\eta_t}{1-2\eta_t} - \frac{\eta_{t+1}}{1-2\eta_{t+1}} \right)  - 4(\sqrt{T\Delta\log K}+\log K)\cdot \frac{\eta_T}{1-2\eta_T} \\
&\stepb{\ge} -16\sum_{t=1}^{T-1}(\sqrt{t\Delta\log K}+\log K)\cdot (\eta_t - \eta_{t+1}) - 16(\sqrt{T\Delta\log K}+\log K)\cdot \eta_T\\
&\stepc{\ge} -16 \sum_{t=1}^{T-1}\sqrt{t\Delta\log K}\cdot \sqrt{\frac{\log K}{\Delta t^3}} - 16\sqrt{T\Delta\log K}\cdot \sqrt{\frac{\log K}{\Delta T}} - 16\eta_1\log K\\
&\ge -16(1+\log T)\log K - 32\log K \\
&= -16(3+\log T)\log K,
\end{align*}
where (a) follows from \eqref{eq:regret_bound_initial}, $\eta_{t+1}\le \eta_t\le 1/4$ and the increasing property of $x\in [0,1/4]\mapsto x/(1-2x)$, (b) is due to the elementary inequality 
\begin{align*}
\frac{x}{1-2x} - \frac{y}{1-2y} = \frac{x-y}{(1-2x)(1-2y)} \le 4(x-y)
\end{align*}
for $1/4\ge x\ge y>0$, and (c) follows from the choice of $\eta_t$ and $t^{-1/2} - (t+1)^{-1/2} \le t^{-3/2}$. Hence, now choosing $t=T$ in \eqref{eq:regret_bound_initial}, we obtain
\begin{align*}
&\max_{a\in [K]} \sum_{s\le t} r_{s,a} - \sum_{t=1}^T\sum_{a=1}^K p_{t,a}r_{t,a}\\ &\le 4\sqrt{T\Delta\log K} +4\log K- 2\sum_{t=1}^T \eta_t\sum_{a=1}^K p_{t,a}r_{t,a} \nonumber\\
&\le  4\sqrt{T\Delta\log K} + 32(4+\log T)\log K, 
\end{align*}
giving the second statement of Theorem \ref{thm:good_expert}. Since it is further upper bounded by $132\sqrt{T\Delta\log K}$ as long as $\Delta\ge T^{-1}(1+\log T)^2\log K$, we arrive at the claimed upper bound.

\subsubsection{Proof of the Lower Bound}
The main result of this section is to show that the $\Theta(\sqrt{T\Delta\log K})$ regret is minimax rate-optimal, with the constant $c=1/16$ in Theorem \ref{thm:good_expert}. 

The lower bound proof relies on a standard application of testing multiple hypotheses, where the learner cannot distinguish between a carefully designed class of reward distributions and therefore incurs a large regret. Specifically, consider the following class of expert rewards $(r_{t,a})_{t\in[T], a\in[K]}$: for each $t\in [T]$, the reward vector $(r_{t,a})_{a\in [K]}$ is random and follows the following joint distribution: 
\begin{enumerate}
	\item with probability $\frac{1}{2}$, the reward vector is $(1-\Delta, 1, \cdots, 1)$; 
	\item with probability $\frac{1-4\Delta}{2(1-2\Delta)}$, the reward vector is $(0,0,\cdots,0)$; 
	\item with remaining probability $\frac{\Delta}{1-2\Delta}$, the reward vector is $(1-\Delta, r_2, \cdots, r_K)$, where $r_i \sim \mathsf{Bern}(p_i)$ for $i=2,3,\cdots,K$ are mutually independent. 
\end{enumerate}
Moreover, the rewards across different times are mutually independent. To ensure all probabilities lie in $[0,1]$, we assume throughout this subsection that $\Delta\le 1/4$, as a smaller $\Delta$ always makes the minimax regret smaller and thus suffices for the proof of the lower bound. The construction above satisfies two properties: first, expert $1$ is always $\Delta$-good; second, the expected reward of the expert $i\in [K]$ is
\begin{align}\label{eq:mean_reward}
\bE[r_{t,i}] = \begin{cases}
\frac{1-\Delta}{2(1-2\Delta)} & \text{if } i = 1, \\
\frac{1-\Delta}{2(1-2\Delta)} + \frac{\Delta}{1-2\Delta}\left(p_i - \frac{1}{2}\right) & \text{if } 2\le i\le K.
\end{cases}
\end{align}
As a result, the vector $(p_2,\cdots,p_K)$ modulates the reward information of the experts. 

Next, we choose $K$ different sets of parameters $(p_2,\cdots,p_K)$ that correspond to different scenarios, with the following specific choices: 
\begin{align*}
p_i^{(j)} = \frac{1}{2} - \delta + 2\delta\cdot \1(i=j), \qquad i\in \{2,3,\cdots,K\}, j\in [K], 
\end{align*}
where $\delta\in (0,1/4)$ is some parameter yet to be chosen. The main properties of the above construction are as follows: 
\begin{enumerate}
	\item For $j\in [K]$, expert $j$ has the highest expected reward in the $j$-th scenario, and choosing any other expert incurs an instantaneous (pseudo-)regret at least $2\Delta\delta$ according to \eqref{eq:mean_reward}; 
	\item For $j=2,\cdots,K$, the $j$-th scenario differs from the first one only through the choice of $p_j$. 
\end{enumerate}
Let $\bE^{(j)}$ denote the expectation under the $j$-th scenario, and $\bP^{(j)}$ denote the corresponding expectation. Then for any policy $\pi=(a_1,\cdots,a_T)$, we have
\begin{align}\label{eq:lower_bound}
\sup_{(r_{t,a})} R_T(\pi) &\stepa{\ge} \frac{1}{K}\sum_{j=1}^K \bE^{(j)}[R_T(\pi)] \nonumber\\
&= \frac{1}{K}\sum_{j=1}^K \sum_{t=1}^T \bE^{(j)}\left[\max_{a\in [K]}r_{t,a} - r_{t,a_t} \right] \nonumber\\
&\stepb{\ge} \frac{1}{K}\sum_{j=1}^K  \sum_{t=1}^T \left[\max_{a\in [K]}\bE^{(j)}[r_{t,a}] - \bE^{(j)}[r_{t,a_t}] \right] \nonumber\\
&\stepc{\ge} \frac{1}{K}\sum_{j=1}^K  \sum_{t=1}^T 2\Delta\delta\cdot \bP^{(j)}(a_t \neq j) \nonumber\\
&\stepd{\ge} 2T\Delta \cdot \delta \left(1 - \frac{I(V;X) + \log 2}{\log K}\right),
\end{align}
where (a) follows from the fact that the maximum is no smaller than the average, (b) follows from the linearity of expectation and the inequality $\bE[\max_n X_n]\ge \max_n \bE[X_n]$, (c) follows from the first property of the choice of $\bP^{(j)}$, and (d) is due to the Fano's inequality (cf. Lemma \ref{lemma.fano} in the appendix) applied to $V \sim \mathsf{Unif}([K])$, $X = (r_{t,a})_{t\in[T], a\in[K]}$, and $P_{X|V=j} = \bP^{(j)}$. By \eqref{eq:lower_bound}, it suffices to prove an upper bound of the mutual information $I(V;X)$. Using the variational representation of the mutual information $I(V;X) = \min_{Q_X} \bE_V[D_{\text{KL}}(P_{X|V} \| Q_X)]$, we have
\begin{align}\label{eq:mutual_info}
I(V;X) &\le \frac{1}{K}\sum_{j=1}^K D_{\text{KL}}(\bP^{(j)} \| \bP^{(1)}) \nonumber\\
&\stepa{=} \frac{1}{K}\sum_{j=1}^K \sum_{t=1}^T D_{\text{KL}}(\bP_t^{(j)} \| \bP_t^{(1)}) \nonumber\\
&\stepb{\le}  \frac{T\Delta}{1-2\Delta} D_{\text{KL}}( \mathsf{Bern}(\frac{1}{2} + \delta) \| \mathsf{Bern}(\frac{1}{2} - \delta)) \nonumber \\
&\stepc{\le} 2T\Delta\cdot 32\delta^2,
\end{align}
where (a) is due to the chain rule of the KL divergence where $\bP_t^{(j)}$ denotes the distribution of rewards at time $t$ in the $j$-th scenario, (b) follows from the second property of $\bP^{(j)}$ and the data processing inequality, and (c) follows from $1-2\Delta\ge 1/2$ and $D_{\text{KL}}( \mathsf{Bern}(1/2 + \delta) \| \mathsf{Bern}(1/2 - \delta))\le 32\delta^2$ as long as $\delta\le 1/4$. Finally, combining \eqref{eq:lower_bound}, \eqref{eq:mutual_info} and choosing
\begin{align*}
\delta = \frac{1}{16}\sqrt{\frac{\log K}{T\Delta}} \le \frac{1}{4}
\end{align*}
gives the claimed lower bound in Theorem \ref{thm:good_expert} with $c=1/16$.  

\subsection{Proof of Lemma \ref{lemma:continuity}}
We show that for any $v,m \in [0,1]$ and $b\le \min\{v,b'\}$, it holds that 
\begin{align}\label{eq:right_lipschitz}
	r(b;v,m) \le r(b';v,m) + (b' - b). 
\end{align}
In fact, straightforward algebra gives
\begin{align*}
	r(b';v,m) &= (v - b')\1(b'\ge m) \\
	&\ge (v-b)\1(b'\ge m) - (b' - b)\\
	&\ge (v-b)\1(b\ge m) - (b' - b) \\
	&= r(b;v,m) - (b'-b), 
\end{align*}
establishing \eqref{eq:right_lipschitz}. As a result, since $f(v)\le g(v)$ for all $v\in [0,1]$, it holds that
\begin{align*}
	r_t(f) &= r(f(v_t);v_t,m_t) \\
	&\le r(g(v_t);v_t,m_t) + (g(v_t) - f(v_t)) \\
	&\le r_t(g) + \|f-g\|_\infty. 
\end{align*}
Now summing over $t=1,2,\cdots,T$ gives the desired result.

\subsection{Proof of Theorem \ref{thm:chew}}
For any $m\in [M]$ and $f_{m-1}\in \calN_{m-1}$, the mixture distribution $P_{t,f_{m-1}}$ used by the manager $f_{m-1}$ is in turn a mixture over her employees, with the mixture $Q_{t,f_{m-1}}$ proportional to the exponential weights of the past cumulative rewards of the employees. Moreover, among the employees there is a $\Delta_{m-1}$-good expert $f_{m-1}^\star$. Hence, the regret upper bound of Theorem \ref{thm:good_expert} implies that for every $f_m\in \calN_m, f_{m-1}\in \calN_{m-1}$ with $f_m \to f_{m-1}$, 
\begin{align}\label{eq:manager_regret_bound}
	&\sum_{t=1}^T\left( \bE_{f_t\sim P_{t,f_m}}[r_t(f_t)] - \bE_{f_t\sim P_{t,f_{m-1}}}[r_t(f_t)] \right)\nonumber \\
	&= \sum_{t=1}^T\left( \bE_{f_t\sim P_{t,f_m}}[r_t(f_t)] - \bE_{f_m\sim Q_{t,f_{m-1}}} \bE_{f_t\sim P_{t,f_{m}}}[r_t(f_t)] \right) \nonumber\\
	&\le C\left( \sqrt{\frac{C_{\text{\rm Lip}}T\Delta_{m-1}}{\varepsilon_m}}  + (1+\log T)\cdot \frac{C_{\text{\rm Lip}}}{\varepsilon_m} \right). 
\end{align}
Using the telescoping in \eqref{eq:regret_decomposition}, we have
\begin{align*}
 R_T(\pi^{\text{\rm ChEW}}) 
	&\le T\varepsilon_M + \sum_{m=1}^M C\left( \sqrt{\frac{C_{\text{\rm Lip}}T\Delta_{m-1}}{\varepsilon_m}} + (1+\log T)\cdot \frac{C_{\text{\rm Lip}}}{\varepsilon_m} \right) \nonumber \\
	&\le 2\sqrt{T} + \sum_{m=1}^M C\left(\sqrt{\frac{C_{\text{\rm Lip}}T\cdot 2^{-m+3}}{2^{-m}}} + (1+\log T)\cdot \frac{C_{\text{Lip}}}{2^{-m}}\right) \nonumber \\
	&\le \left(2 + C\sqrt{2C_{\text{\rm Lip}} }\log T + 2CC_{\text{\rm Lip}}(1+\log T)\right)\sqrt{T}, 
\end{align*}
where we note that by our choice of $\varepsilon_m = 2^{-m}$, it holds that $\Delta_m = 2\sum_{r=m}^{M-1} \varepsilon_r \le 2\sum_{r=m}^{M-1} 2^{-r} < 2^{-m+2}$. This establishes Theorem \ref{thm:chew}. 

\subsection{Proof of Lemma \ref{lemma:employee}}
We distinguish into three cases. If the image set $f(I_{\ell+1,m'})$ include some element in $J_{\ell+1,2u+1}\backslash J_{\ell+1,2u}$, then by the $1$-Lipschitzness of $f$ and $|I_{\ell+1,m'}|=2^{-\ell-2}$, the total variation of $f$ in $I_{\ell+1,m'}$ is at most $2^{-\ell-2}$. Hence, by construction of the intervals in \eqref{eq:interval_J}, we conclude that $f(I_{\ell+1,m'}) \subseteq J_{\ell+1,2u+1}$. Similarly, if $f(I_{\ell+1,m'})$ include some element in $J_{\ell+1,2u-1}\backslash J_{\ell+1,2u}$, we will have $f(I_{\ell+1,m'}) \subseteq J_{\ell+1,2u-1}$. If neither of the above cases holds, it is then clear that $f(I_{\ell+1,m'}) \subseteq J_{\ell+1,2u}$, as desired. 

\subsection{Proof of Theorem \ref{thm:SEW}}
We first analyze the complexity of the SEW policy. For the space complexity, the initialization step needs to keep track of all cumulative rewards of all experts, which takes 
\begin{align*}
	\sum_{\ell = 1}^L M_{\ell}(U_{\ell} + W_{\ell}) = O\left(\sum_{\ell=1}^L 2^{2\ell} \right) = O\left(2^{2L}\right) = O(T)
\end{align*}
space. As for other steps, only some temporary variables need to be stored each time for the unique bin at each level, and therefore they require $O(\sum_{\ell=1}^L (U_{\ell} + W_{\ell}) ) = O(2^L) = O(\sqrt{T})$ additional space. Hence, the overall space complexity of the algorithm is $O(T)$. 

Next we turn to the time complexity. As before, the initialization step only involves the assignments of $O(T)$ variables and thus takes $O(T)$ time. As for the updates at each time, Step 1 of Algorithm \ref{algo:sew} evaluates the EW probability $W_{\ell}$ times at each level $\ell\in [L]$. Since the EW algorithm (cf. Algorithm \ref{algo:ew}) only takes $O(1)$ time to evaluate the probability vector supported on $4$ elements, Step 1 only takes $O(\sum_{\ell=1}^L W_{\ell}) = O(\sqrt{T})$ time at each round. At Step 2, the bidder only needs to sample a random variable at each round, which takes $O(\log T)$ time in total. As for Step $3$, each reward update in \eqref{eq:update_reward} only takes $O(1)$ time, and thus it takes $O(\sum_{\ell=1}^L (U_{\ell} + W_{\ell}) ) = O(\sqrt{T})$ time in total. Hence, the time complexity of Algorithm \ref{algo:sew} at each round is $O(\sqrt{T})$, and the overall time complexity is $O(T^{3/2})$, as claimed. 

Finally we show the claimed upper bound on the regret. It suffices to show that, for any oracle strategy $f\in \calF_{\text{\rm Lip}}$, the reward difference between the strategy $f$ and the SEW policy $\pi^{\text{SEW}}$ is upper bounded by $O(\sqrt{T}\log T)$. To this end, we aim to construct a chain of policies $\pi_0, \pi_1, \cdots, \pi_L, \pi_{L+1}$ with $\pi_0 = \pi^{\text{SEW}}$, $\pi_{L+1} = \{f(v_t)\}_{t\in[T]}$ and the reward differences between adjacent policies are upper bounded. Specifically, based on $f$, let
\begin{align*}
	\widetilde{f}(v) = \frac{1}{M_L}\left\lceil M_L\sup_{(i-1)/M_L < v \le i/M_L} f(v) \right\rceil, 
\end{align*}
for $\frac{i-1}{M_L} < v\le \frac{i}{M_L}$, be an upper approximation of $f$ which is piecewise constant, where we recall that $M_L = 2^{L+1}$ as given in Algorithm \ref{algo:sew}. Clearly, by the $1$-Lipschitz property of $f$, it holds that $\widetilde{f} - 2/M_L \le f \le \widetilde{f}$ everywhere. Hence, let $\pi_L$ be the policy which bids $\widetilde{f}(v_t)$ at time $t$, Lemma \ref{lemma:continuity} gives
\begin{align}\label{eq:regret_last_level}
	\sum_{t=1}^T r(\pi_{L+1}(t); v_t,m_t) - \sum_{t=1}^T r(\pi_{L}(t); v_t,m_t) \le \frac{2T}{M_L} \le 2\sqrt{T}, 
\end{align}
where the last step is due to the choice of $L = \lfloor \log_2 \sqrt{T} \rfloor$. 

\begin{figure}[!t]
	\centering
	\begin{tikzpicture}
		\draw [thick, <->] (0,9) -- (0,0) -- (10,0); 
		\node [below] at (10,0) {private value}; \node [left] at (0,9) {bid}; 
		\draw [dashed] (1,0) -- (1,8); 	\draw [dashed] (2,0) -- (2,8); 	\draw [dashed] (3,0) -- (3,8); 
		\draw [dashed] (4,0) -- (4,8); 	\draw [dashed] (5,0) -- (5,8); 	\draw [dashed] (6,0) -- (6,8);
		\draw [dashed] (7,0) -- (7,8); 	\draw [dashed] (8,0) -- (8,8);  
		\draw [dashed] (0,1) -- (8,1); 	\draw [dashed] (0,2) -- (8,2); 
		\draw [dashed] (0,3) -- (8,3); 	\draw [dashed] (0,4) -- (8,4);
		\draw [dashed] (0,5) -- (8,5); 	\draw [dashed] (0,6) -- (8,6);
		\draw [dashed] (0,7) -- (8,7); 	\draw [dashed] (0,8) -- (8,8);
		\node [below] at (8,0) {$1$}; \node [left] at (0,8) {$1$}; \node [below] at (0,0) {$0$};   
		\draw [thick] (0,2) -- (1,2) -- (1,3) -- (2,3) -- (2,4) -- (3,4) -- (3,5) -- (4,5) -- (4,6) -- (5,6) -- (5,7) -- (6,7) -- (6,8) -- (7,8) -- (7,7) -- (8,7);
		\node [right] at (8,7) {$\widetilde{f}$}; 
		\draw [thick, blue] (0.05, 1.05) rectangle (0.95, 2.95);
		\draw [thick, blue] (1.05, 2.05) rectangle (1.95, 3.95); 
		\draw [thick, blue] (2.05, 3.05) rectangle (2.95, 4.95);  
		\draw [thick, blue] (3.05, 4.05) rectangle (3.95, 5.95);  
		\draw [thick, blue] (4.05, 5.05) rectangle (4.95, 6.95);  
		\draw [thick, blue] (5.05, 6.05) rectangle (5.95, 7.95);  
		\draw [thick, blue] (6.05, 7.95) -- (6.95, 7.95); 
		\draw [thick, blue] (7.05, 6.05) rectangle (7.95, 7.95);
		\draw [thick, red] (0.01,0.01) rectangle (1.99, 3.99);   
		\draw [thick, red] (2.01,2.01) rectangle (3.99, 5.99);   
		\draw [thick, red] (4.01,4.01) rectangle (5.99, 7.99);  
		\draw [thick, red] (6.01,4.01) rectangle (7.99, 7.99);   
		\draw [thick, dotted] (0,1.9) -- (1,1.7) -- (2,2.6) -- (3,3.4) -- (4,4.4) -- (5,5.3) -- (6,6.3) -- (6.8, 7.1) -- (7, 6.9) -- (8, 6.1);
		\node [right] at (8, 6.1) {$f$}; 
		\node [below, blue] at (7.5, 6) {$\pi_2$}; \node [below, red] at (6.5,4) {$\pi_1$}; 
	\end{tikzpicture}
	\caption{An example of the chain of policies with $L=2$. The dotted and solid lines represent the $1$-Lipschitz strategy $f$ and its piecewise constant approximation $\widetilde{f}$, respectively. The small blue rectangles (or segments for dummy experts) represent the level-$2$ experts associated with $\widetilde{f}$ on each bin, and the large red rectangles represent the level-$1$ experts associated with the previous level-$2$ experts.} \label{fig.chain}
\end{figure}

It then remains to specify the policies $\pi_1, \cdots, \pi_{L-1}$. First, by the $1$-Lipschitz property of $f$, it is easy to see that the value difference of $\widetilde{f}$ at adjacent pieces is at most $1/M_L$. Moreover, each piece of $\widetilde{f}$ is associated with a level-$L$ expert on each bin. Therefore, following the same lines as the proof of Lemma \ref{lemma:employee}, any pair of the previous level-$L$ experts on each level-$(L-1)$ bin is an employee of a level-$(L-1)$ manager\footnote{This holds true even if a level-$L$ expert is a dummy expert, while the level-$(L-1)$ manager is never dummy. See also the seventh piece of Figure \ref{fig.chain}.}. Consequently, we further have a sequence of level-$(L-1)$ experts on each bin, and this process can be continued until we reach level $1$. An example with $L=2$ is illustrated in Figure \ref{fig.chain}. Now let $\pi_{\ell}$ be the policy which follows the previously chosen level-$\ell$ experts inside each level-$\ell$ bin, and it is clear that in this way $\pi_{L}$ exactly bids the same price as $\widetilde{f}$. 

To upper bound the reward difference between each adjacent levels, we utilize the key property that the experts followed by the policy $\pi_{\ell+1}$ in each bin are employees of the expert followed by the (random) policy $\pi_{\ell}$. Moreover, due to the presence of the dummy employee at each level, any level-$\ell$ manager has a good employee with suboptimality gap at most $2^{-\ell}$. Hence, for each $\ell\in \{0,1,\cdots,L-1\}$, we have (recall that $r_t(b):=r_t(b;v_t,m_t)$)
\begin{align}\label{eq:regret_middle_level}
	\sum_{t=1}^T \left( r_t(\pi_{\ell+1}(t)) - \bE[r_t(\pi_{\ell}(t))] \right) &= \sum_{m=1}^{M_{\ell+1}}\sum_{t\in [T]: v_t\in I_{\ell+1,m}}\left( r_t(\pi_{\ell+1}(t)) - \bE[r_t(\pi_{\ell}(t))] \right) \nonumber \\
	&\stepa{\le} \sum_{m=1}^{M_{\ell+1}} C\left(\sqrt{T_{\ell+1,m}\cdot 2^{-\ell}\log 4} + (1+\log T)\cdot \log 4\right) \nonumber\\
	&\stepb{\le} C\sqrt{M_{\ell+1}T\cdot 2^{-\ell}\log 4 } + C(1+\log T)M_{\ell+1}\cdot \log 4 \nonumber\\
	&\le 4C\sqrt{T} + 4C(1+\log T)2^{\ell}, 
\end{align}
where (a) is due to that the policy $\pi_{\ell}(t)$ is running an independent EW algorithm on the subsequence $v_t\in I_{\ell+1,m}$ for each $m\in [M_{\ell+1}]$, where $\pi_{\ell+1}(t)$ is one employee, and therefore Theorem \ref{thm:good_expert} can be applied with $T_{\ell+1,m} = \sum_{t=1}^T \1(v_t\in I_{\ell+1,m})$. The inequality (b) follows from the concavity of $x\mapsto \sqrt{x}$, and the last inequality plugs in the expression of $M_{\ell} = 2^{\ell+1} \le 2\sqrt{T}$. 

Finally, by \eqref{eq:regret_last_level} and \eqref{eq:regret_middle_level}, it holds that
\begin{align*}
	\sum_{t=1}^T \left( r_t(f(v_t)) - \bE[r_t(\pi^{\text{SEW}}(t))] \right) 
	&=\sum_{\ell=0}^{L-1} \sum_{t=1}^T \left( r_t(\pi_{\ell+1}(t)) - \bE[r_t(\pi_{\ell}(t))] \right) + \sum_{t=1}^T (r_t(\pi_{L+1}(t)) - r_t(\pi_{L}(t))) \\
	&\le 4CL\sqrt{T} + 4C(1+\log T)\sum_{\ell=0}^{L-1}2^{\ell} + 2\sqrt{T} \\
	&\le \left(2+4C(1+2\log_2 T)\right)\sqrt{T}, 
\end{align*}
and the arbitrariness of $f\in \calF_{\text{Lip}}$ completes the proof of Theorem \ref{thm:SEW}. 

\subsection{Proof of Theorem \ref{thm:monotone}}
The construction is essentially the same as \cite[Theorem 8]{han2020optimal}, but serves for different purposes. Fix $v_1 = 1/2$, and let $m_t$ be an i.i.d. sequence with $m_1$ uniformly distributed on two points $\{0,1/8\}$. The rest of the private values $v_t$ are chosen sequentially as follows: for $t\ge 2$, 
\begin{align*}
	v_t = \begin{cases}
		v_{t-1} + 2^{-t-1} & \text{if } m_{t-1} = 0, \\
		v_{t-1} - 2^{-t-1} & \text{if } m_{t-1} = 1/8. \\
	\end{cases}
\end{align*}
As $\sum_{t\ge 2}2^{-t-1} = 1/4$, we have $v_t \in [1/4,3/4]$ for all $t\in [T]$. Moreover, the update rule of $v_t$ ensures that if $v_{t+1} > v_t$, then $v_s>v_t$ for all $s>t$. Symmetrically, if $v_{t+1}<v_t$, then $v_s<v_t$ for all $s>t$. Based on this observation, we claim that the best strategy of the monotone oracle is 
\begin{align*}
	f(v) = \frac{1}{8}\cdot \1(v>v_T) + m_T\cdot \1(v=v_T). 
\end{align*}
In fact, as $v_t \ge 1/4> m_t$, the best bid (of any oracle, not necessarily restricted to be monotone) is clearly $m_t$ at each time. We now claim that $f(v_t) = m_t$ for the above strategy $f$. In fact, if $v_t>v_T$, the update rule of $v_t$ show that $v_t > v_{t+1}$, so it must hold that $m_t = 1/8 = f(v_t)$. Similarly, if $v_t<v_T$, it must hold that $m_t = 0$, which is also $f(v_t)$. Finally, $f(v_T) = m_T$ clearly holds. Hence, the strategy $f$ is the best monotone strategy, and the cumulative reward is $\sum_{t=1}^T (v_t - m_t)$. 

Next we consider any bidding strategy used by the bidder. Since $m_t$ is totally unpredictable based on the history, the expected reward at each time is at most 
\begin{align*}
	\bE[r(b_t;v_t,m_t)] &= \bE[(v_t - b_t)\1(b_t\ge m_t)] \\
	&\le \max_{b\in [0,1]} \frac{v_t - b}{2}\left(1 + \1\left(b\ge \frac{1}{8}\right)\right) \\
	&= \max\left\{\frac{v_t}{2}, v_t - \frac{1}{8} \right\} = v_t - \frac{1}{8}, 
\end{align*}
where the last identity is due to $v_t \ge 1/4$. Hence, the reward difference between the monotone oracle and the bidder is at least $\sum_{t=1}^T (1/8 - m_t)$, which has expectation $T/16 = \Omega(T)$. 

\subsection{Proof of Theorem \ref{thm:monotonce_oracle}}
Let $p = q/(q-1) \in [1,\infty)$ be the H\"{o}lder conjugate of $q\in (1,\infty]$. If two functions $f,g \in \calF_{\text{\rm Mono}}$ satisfy $f\le g$ and $\|f-g\|_p \le \varepsilon$, then by H\"{o}lder's inequality, 
\begin{align*}
	&\bE\left[ r(f(v_t);v_t,m_t) - r(g(v_t);v_t,m_t) \right] \\
	&\le \bE_{P_t}|f(v_t) - g(v_t)| \\
	&= \|(f-g)p_t\|_1 \le \|f-g\|_p\|p_t\|_q \le L\varepsilon. 
\end{align*}
As a result, we may treat the $\varepsilon$-bracketing of $\calF_{\text{\rm Mono}}$ under the $L_p$ norm as the experts, where the same approximation property still holds in expectation. Then applying the ChEW policy with the new collection of the experts gives the claimed regret bound. 
\bibliographystyle{alpha}
\bibliography{di.bib}

\end{document}